\documentclass[10pt,twocolumn,letterpaper]{article}

\usepackage{cvpr}              
\usepackage{cuted}
\usepackage[accsupp]{axessibility}

\newcommand{\tangyu}[1]{{#1}}

\newcommand{\method}{ComRoPE}

\newcommand{\switch}[2]{#2}

\usepackage{amsthm}
\usepackage{float}
\restylefloat{table}

\usepackage{placeins}
\usepackage{amsmath}

\newtheorem{theorem}{Theorem}
\newtheorem{lemma}{Lemma}

\newtheorem{definition}{Definition}

\newtheorem{proposition}{Proposition}
\newtheorem{remark}{Remark}

\usepackage{algorithm}
\usepackage{algorithmic}

\usepackage{makecell}
\usepackage{multirow}

\usepackage[normalem]{ulem}
\useunder{\uline}{\ul}{}

\usepackage{amsmath,amsfonts,bm}

\def\eqref#1{equation~\ref{#1}}

\def\1{\bm{1}}

\DeclareMathAlphabet{\mathsfit}{\encodingdefault}{\sfdefault}{m}{sl}
\SetMathAlphabet{\mathsfit}{bold}{\encodingdefault}{\sfdefault}{bx}{n}

\definecolor{cvprblue}{rgb}{0.21,0.49,0.74}
\usepackage[pagebackref,breaklinks,colorlinks,allcolors=cvprblue]{hyperref}

\def\paperID{18106} 
\def\confName{CVPR}
\def\confYear{2025}

\title{\method: Scalable and Robust Rotary Position Embedding Parameterized by Trainable Commuting Angle Matrices}

\author{
    \textbf{Hao Yu}$^{1}$  \quad
    \textbf{Tangyu Jiang}$^{1\dagger}$  \quad
    \textbf{Shuning Jia}$^{1,2}$ \quad
    \textbf{Shannan Yan}$^{1}$  \quad
    \textbf{Shunning Liu}$^{1}$  \\
    \textbf{Haolong Qian}$^{1}$  \quad
    \textbf{Guanghao Li}$^{1}$  \quad
    \textbf{Shuting Dong}$^{1}$ \quad
    \textbf{Huaisong Zhang}$^{1}$ \quad
    \textbf{Chun Yuan}$^{1\dagger}$ \quad
    \\[.6em]
    $^1$Tsinghua University \quad
    $^2$Shenzhen University 
    \\[.5em]
    {\tt\small 
        longinyh@gmail.com, {jiangtangyu,yuanc}@sz.tsinghua.edu.cn
    }
}

\begin{document}
\maketitle
\newcommand\blfootnote[1]{%
  \begingroup
  \renewcommand\thefootnote{}\footnote{#1}%
  \addtocounter{footnote}{-1}%
  \endgroup
}
\blfootnote{$^\dagger$Corresponding author. This work was done when Shuning Jia was an intern at Tsinghua University.}

\vspace{-2em}
\begin{abstract}
The Transformer architecture has revolutionized various fields since it was proposed, where positional encoding plays an essential role in effectively capturing sequential order and context.
Therefore, Rotary Positional Encoding (RoPE) was proposed to alleviate these issues, which integrates positional information by rotating the embeddings in the attention mechanism.
However, RoPE utilizes manually defined rotation matrices, a design choice that favors computational efficiency but limits the model's flexibility and adaptability.
In this work, we propose \method{}, which generalizes \textbf{RoPE} by defining it in terms of trainable \textbf{com}muting angle matrices.
Specifically, we demonstrate that pairwise commutativity of these matrices is essential for RoPE to achieve scalability and positional robustness. 
We formally define the RoPE Equation, which is an essential condition that ensures consistent performance with position offsets. 
Based on the theoretical analysis, we present two types of trainable commuting angle matrices as sufficient solutions to the RoPE equation,
which significantly improve performance, surpassing the current state-of-the-art method by 1.6\% at training resolution and 2.9\% at higher resolution on the ImageNet-1K dataset.
Furthermore, our framework shows versatility in generalizing to existing RoPE formulations and offering new insights for future positional encoding research. 
To ensure reproducibility, the source code and instructions are available at \url{https://github.com/Longin-Yu/ComRoPE}.
\end{abstract}
    
\vspace{-1.5em}
\section{Introduction}
\label{sec:intro}









The Transformer architecture \citep{vaswani2017attention} has been widely adopted across various fields, including Natural Language Processing (NLP) \citep{Devlin2019BERTPO, radford2018improving, radford2019language, brown2020language, glm2024chatglm} and Computer Vision (CV) \citep{VIT2020}. Moreover, an increasing number of Transformer-based applications \citep{zhang2024distilling, li2025llava, chen2024gim, zhang2025rethinking, liuagentbench, zhang2025imdprompter, liu2023webglm, zhang2024can, lai2024autowebglm} continue to demonstrate their effectiveness across a wide range of domains.
At the core of the Transformer model lies the attention mechanism, which enables the model to selectively focus on different parts of the input based on relevance, thereby effectively capturing long-range dependencies and contextual relationships.

However, since the attention mechanism is insensitive to the fundamental position information, 
it cannot inherently capture the order of elements in the data. 
In order to alleviate this issue, positional embeddings are added to the input representation, providing the model with the necessary positional information to process the sequence order and structure effectively.

A variety of positional encoding methods have been proposed, which can be divided into two categories: Absolute Positional Encoding (APE) and Relative Positional Encoding (RPE).
APE explicitly encodes the absolute position of each token by generating fixed positional embeddings, which are directly added to the input embeddings at the beginning of training \cite{vaswani2017attention}. However, APE struggles with handling long sequences and exhibits high sensitivity to positional shifts, limiting the scalability and robustness of models \cite{kazemnejad2023impact}. In contrast, RPE does not directly modify input embeddings; instead, it incorporates relative positional information within the attention mechanism, enabling the model to effectively capture positional relationships between tokens.
Among all the existing works, Rotary Position Embedding (RoPE) \cite{su2024roformer} has gained significant attention due to the advantage of applying a rotational transformation to token embeddings. 
It encodes relative positions by treating each pair of features as coordinates and rotating them by an angle proportional to their position, enabling all tokens to interact within the attention mechanism regardless of distance. 

However, existing RoPE methods face several key challenges: 
i) The essential components (i.e., the RoPE matrices) of previous RoPE approaches rely on 2D rotation groups, which simplify computations but consequently restrict their feature projection capabilities, especially in high-dimensional spaces \citep{su2024roformer}.
ii) Moreover, the majority of the rotation matrices of RoPE require to be manually designed, leading to insufficient capability and suboptimal performance\citep{press2021train,raffel2020exploring}.
iii) Finally, previous attempts \citep{ostmeier2024liere} to extend the rotation group often prioritize design simplicity, making it difficult to consistently satisfy relative position dependency—a critical property of RoPE that ensures positional robustness against offsets.

\vspace{.5em}
\noindent\textbf{Our objectives.}
This work aims to develop a novel RoPE method for Transformers that is both scalable and robust. 
Specifically, we seek to extend the rotation group from the existing 2D representation to a larger subgroup of the special orthogonal group, allowing for higher degrees of freedom 
while preserving consistent behavior with respect to position offsets.
Unlike existing methods that rely on manually designed non-trainable rotation matrices, which suffer from limited expressiveness and reduced robustness, 
our approach is designed to offer richer feature representation capabilities.
This framework addresses the scalability limitations of current approaches and enhances their robustness against positional transformations.


\vspace{.5em}

\noindent\textbf{Our contributions.} This work introduces \method{}, a novel framework that significantly enhances positional encoding in Transformers. \method{} leverages trainable angle matrices, extending the RoPE mechanism with higher scalability and robustness. We identify the pairwise commutativity of these matrices as a necessary and sufficient condition for effective positional encoding, thereby unifying various existing RoPE formulations under a single theoretical framework. The contributions are summarized as follows:

\begin{itemize}
   \item We formally define the RoPE function parameterized by angle matrices and prove that pairwise commutativity is a necessary and sufficient condition, offering a unified theory that encompasses several existing RoPE variants.
   \item We introduce \method{}, a scalable and robust solution that leverages two types of trainable commuting angle matrix sets as sufficient solutions to the RoPE equation, capturing richer positional representations without the need for manual design.
   \item Our extensive experiments show that \method{} surpasses the current state-of-the-art LieRE, achieving a performance increase of $1.6\%$ at training resolution and $2.9\%$ at higher resolutions on the ImageNet-1K classification task while delivering strong results across other benchmarks.
   \item We explore further applications of \method{}, providing valuable insights for advancing position encoding techniques in future Transformer-based models.
\end{itemize}

\vspace{-.4em}
\section{Related work}
\label{sec:formatting}


\subsection{Position information in attention}

Transformers \citep{vaswani2017attention} utilize the attention mechanism to capture similarities within sequences. However, they lack inherent sequential information and cannot capture the positional information of each token. To address this limitation, positional encoding \citep{Devlin2019BERTPO,T52020,shaw2018self} was introduced. As research progressed, positional encoding generally evolved into two types: APE \citep{vaswani2017attention,sukhbaatar2015end,Devlin2019BERTPO} and RPE \citep{shaw2018self,su2024roformer,horn2021translational,liutkus2021relative}.
\citep{vaswani2017attention} first proposed using APE in the form of sine and cosine functions, effectively representing the positional relationships within the input sequence. This positional encoding method achieved remarkable results in natural language processing, becoming a foundational component for many NLP tasks. However, the fixed nature of positional encoding limited the model's ability to generalize to longer input sequences. In response, subsequent research introduced learnable absolute positional encoding. \citep{Devlin2019BERTPO} proposed further enhancing the model's expressive power by incorporating learnable position embeddings, particularly excelling in tasks such as sentence alignment and context representation.
Although APE provides positional information to enhance the model's understanding of sequences, it shows limitations in handling long sequences and cross-sequence scenarios. Hierarchical ViT such as Swin Transformer \citep{swintransformer}, introduced Relative Position Bias (RPB) \citep{swintransformer,raffel2020exploring,shaw2018self} to handle large numbers of tokens with limited positional embeddings \citep{heo2024rotary}. Related research has explored alternative encoding methods, such as RPE, as a replacement for APE to better capture complex dependency structures. iRPE \citep{wu2021rethinking} proposed an improved RPB by combining relative position embedding.





\subsection{Rotary position embedding}
Building on these developments, RoFormer \citep{su2024roformer} combined the advantages of APE and RPE, proposing RoPE. Currently, RoPE is widely used in Large Language Models (LLMs), such as LLaMA \citep{touvron2023llama} and Vicuna \citep{chiang2023vicuna}. This approach enhances model performance on tasks involving long-text semantics and multi-turn dialogues, improving extrapolation capability.
To better adapt to the characteristics of two-dimensional data such as images, researchers extended RoPE to two-dimensional sequences (2D-RoPE). LieRE \citep{ostmeier2024liere} extends RoPE to a more generalized form by introducing a rotation-based positional encoding method grounded in Lie group theory. Unified-IO 2 \citep{lu2024unified} applies 2D-RoPE within its multimodal architecture; EVA-02 \citep{fang2024eva}, FiT \citep{fang2024eva} these pioneering works used 2D RoPE with axial frequencies (2D Axial RoPE), but had limitations in processing in the diagonal direction. Therefore, RoPE for ViT \citep{heo2024rotary} proposes to use mixed axial frequency for 2D RoPE, named RoPE-Mixed.
\section{Method}

In this section, we begin by introducing key definitions and reformulating the RoPE paradigm within multi-axial attention mechanisms, which we collectively refer to as the RoPE Equations.
Next, we present our main theorem, which establishes the necessary and sufficient conditions for RoPE functions parameterized by angle matrices.
Finally, we provide solutions to the RoPE Equations based on the propositions that outline the sufficient conditions.



\subsection{Preliminaries}

\tangyu{We use $\mathbf{R}(\;\cdot\;): \mathbb{R}^N \rightarrow \mathbb{R}^{d\times d}$ to denote a matrix-value function. We denote the vectors and matrices as the lower and uppercase bold font, respectively. 
We first recall two fundamental definitions of matrices as follows:

\begin{definition}[Matrix Exponential]
The exponential of a square matrix \( \mathbf{A} \), denoted as \( e^{\mathbf A} \) or \(\exp(\mathbf A)\), is defined using the matrix exponential series such that:
\[
e^\mathbf{A} = \exp(\mathbf{A}) = \sum_{k=0}^{\infty} \frac{\mathbf{A}^k}{k!} = \mathbf I + \mathbf{A} + \frac{\mathbf{A}^2}{2!} + \frac{\mathbf{A}^3}{3!} + \cdots,
\]
This series converges for any square matrix \( \mathbf{A} \).
\end{definition} 

\begin{definition}[Commuting Matrices]
Two square matrices \( \mathbf{A}_1 \) and \( \mathbf{A}_2 \) are said to \textbf{commute} if their product is independent of the order of multiplication, i.e.,
   \[
   \mathbf{A}_1 \mathbf{A}_2 = \mathbf{A}_2 \mathbf{A}_1.
   \]
A set of square matrices \( \{\mathbf{A}_1, \mathbf{A}_2, \cdots, \mathbf{A}_N\} \) is said to \textbf{pairwise commute} if every pair of matrices within the set commutes with each other. That is, for all \( i, j \) such that \( 1 \leq i, j \leq N \),
   \[
   \mathbf{A}_i \mathbf{A}_j = \mathbf{A}_j \mathbf{A}_i.
   \]
\end{definition}

To clarify and better illustrate the main theorems presented in the following sections, we first reformulate and unify the definitions of the RoPE paradigm, specifically in the context of multi-axial attention mechanisms. RoPE was initially proposed by \citet{su2024roformer} as a positional encoding method based on relative position dependencies. However, previous work provided only conceptual and descriptive descriptions of RPE and RoPE without offering a rigorous formal definition. In this work, we provide the formal definitions of both RPE and RoPE.



\begin{definition}[RPE Equation]
Let  
$f: \mathbb{R}^{d}\times\mathbb{R}^N\rightarrow\mathbb{R}^d$
be a positional encoding function,
and $\rho: \mathbb{R}^{d}\times\mathbb{R}^{d}\rightarrow\mathbb{R}$
be a similarity function.
$f$ is said to be a \textbf{RPE function} if and only if
there exists a function 
$g: \mathbb{R}^{d}\times\mathbb{R}^{d}\times\mathbb{R}^N\rightarrow\mathbb{R}$ 
such that the following conditions hold for all $\bm x, \bm y \in \mathbb{R}^N$ and $\bm q, \bm k \in \mathbb{R}^d$:
\begin{equation}
    \label{eq:rpe-equation}
    g(\bm q,\bm k,\bm x-\bm y)=\rho(f(\bm q,\bm x), f(\bm k, \bm y)),
\end{equation}
We refer to \cref{eq:rpe-equation} as the \textbf{RPE Equation}.
\end{definition}

}


\begin{figure*}[t]
    \centering
    \includegraphics[width=0.85\linewidth]{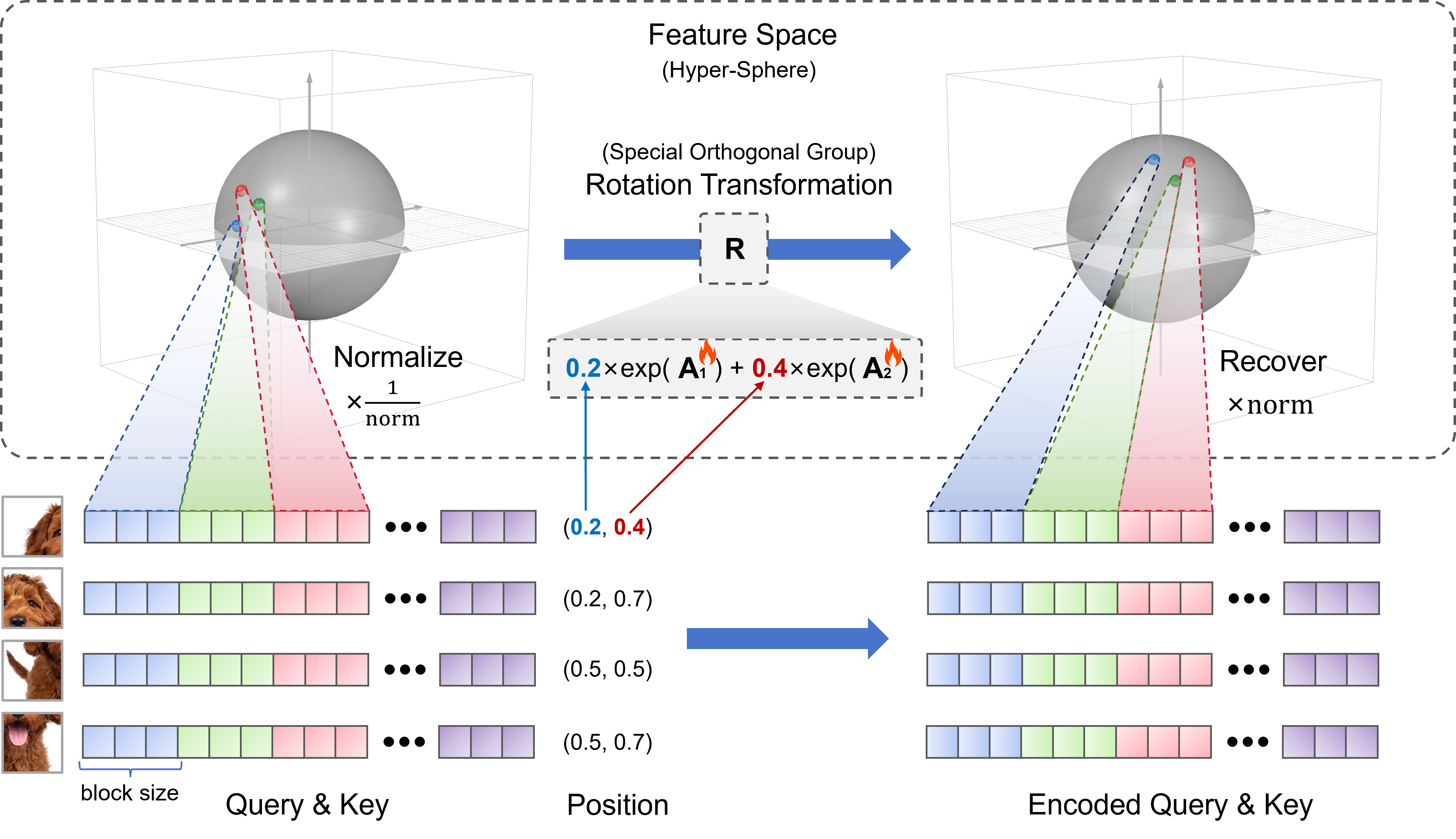}
    \caption{Overview of \method{}. Features are arranged into several blocks, each representing a distinct point in the feature space. The positions, along with the angle matrices, define the rotation matrix, which is an element of the special orthogonal group. The rotation transformation projects a feature point onto another point on the surface of the same hypersphere.}
    \label{fig:overview}
\end{figure*}

\begin{definition}[RoPE Equation]
Let 
$f: \mathbb{R}^{d}\times\mathbb{R}^N\rightarrow\mathbb{R}^d$
be an RPE function.
$f$ is said to be a \textbf{RoPE function} if and only if
there exists a matrix-valued function
$\mathbf{R}_f(\;\cdot\;)$
such that the following conditions hold for all $\bm x, \bm y \in \mathbb{R}^N$ and $\bm q, \bm k \in \mathbb{R}^d$ in RPE Equation:
\begin{equation}
\label{eq:rope-equation}
    \left\{
    \begin{array}{l}
        \vspace{1ex}
        f(\bm q, \bm x) = \mathbf{R}_f(\bm x) \bm q\\
        \vspace{1ex}
        \rho(\bm q, \bm k) = \bm q^\top \bm k \\
        g(\bm q,\bm k, \bm x-\bm y) = \bm q^\top \mathbf{R}_f(\bm y-\bm x) \; \bm k
    \end{array} 
    \right.
\end{equation}
We refer to \cref{eq:rope-equation} as the \textbf{RoPE Equation}.
\end{definition}

By substituting the RoPE equation into the RPE equation, we obtain that the RoPE function satisfies the following property.

\begin{proposition}\label{prop:rope}
\( f \) is said to be a RoPE function if and only if the matrix-valued function \( \mathbf{R}_f(\;\cdot\;) \) satisfies:
\[
\mathbf{R}_f(\bm{x})^\top \mathbf{R}_f(\bm{y}) = \mathbf{R}_f(\bm{y} - \bm{x}),
\]
for all \( \bm{x}, \bm{y} \in \mathbb{R}^N \).
\end{proposition}

\tangyu{Note that the definition of RoPE Equation demonstrates that the position encoding should only be dependent on the relative position of the tokens, which thus be robust against the offset operations. 
We further expand the definition of the RoPE function to a parameterized one (i.e., Definition \ref{def:extend_rope}) via rotation matrices (i.e., Definition \ref{def:extend_rotation_matrix}) as follows\footnote{For clarification, the definitions of the rotation matrix and its exponential representation can be found in Appendix \switch{\ref{app:theorems}}{A}.}.}

\begin{definition}[Parameterized Rotation Matrix]\label{def:extend_rotation_matrix}
Let \( \mathbf{R}(\;\cdot\;; \mathcal A): \mathbb{R}^N \rightarrow \mathbb{R}^{d \times d} \) be a matrix-valued function parameterized by \( N \) skew-symmetric matrices \( \mathcal A = \{ \mathbf{A}_1, \mathbf{A}_2, \ldots, \mathbf{A}_N \} \). We say that \( \mathbf{R}(\;\cdot\;; \mathcal{A}) \) is a \textbf{rotation matrix function} parameterized by angle matrices $\mathcal A$ if it can be expressed as: 
\begin{equation}
    \label{eq:exp-solution}
    \mathbf{R}(\bm{x}; \mathcal{A}) = \exp\left(\sum_{i=1}^{N} \mathbf A_i x_i\right),
\end{equation}
where \( \bm{x} = (x_1, x_2, \ldots, x_N) \in \mathbb{R}^N \) and \( \exp(\;\cdot\;) \) denotes the matrix exponential.
\end{definition}


\begin{definition}[Parameterized RoPE Function]\label{def:extend_rope}
Let 
$f: \mathbb{R}^{d}\times\mathbb{R}^N\rightarrow\mathbb{R}^d$
be a RoPE function.
$f$ is said to be \textbf{parameterized by angle matrices} if and only if
there exists a rotation matrix function $\mathbf{R}_f(\;\cdot\;; \mathcal{A})$ parameterized by angle matrices $\mathcal A$ such that the RoPE Equation (i.e., \cref{eq:rope-equation}) holds.

In this case, $\mathbf{R}_f(\;\cdot\;; \mathcal{A})$ is referred to as the \textbf{rotation matrix function of RoPE function $f(\;\cdot\;; \mathcal{A})$ parameterized by angle matrices $\mathcal A$}. 
\tangyu{For simplicity, we slightly abuse the notation and refer to $\mathbf R$ as the rotation matrix of RoPE function $f$. }
\end{definition}

\subsection{Main theorems}


\tangyu{Based on the formal definitions above, we present our key theoretical results in the following theorem.}

\begin{theorem}\label{theo:main}
\label{the:rotation-commute}
Let \( \mathbf{R}(\;\cdot\;; \mathcal{A}): \mathbb{R}^N \rightarrow \mathbb{R}^{d \times d} \; (N>1)\) be a rotation matrix function parameterized by angle matrices $\mathcal A$.
The rotation difference $\mathbf R(\bm x)^\top\mathbf R(\bm y)$ 
can be represented by 
the location difference $\bm y - \bm x$ 
if and only if  $\mathcal A$ pairwise commute.
\end{theorem}

The proof of Theorem \ref{the:rotation-commute} is shown in Appendix \switch{\ref{app:theorems}}{A}.

Theorem \ref{theo:main} together with Proposition \ref{prop:rope} demonstrate that a function $f$ is a RoPE function parameterized by angle matrices if and only if the angle matrices $\mathcal{A}$ pairwise commute. 
In other words, to construct a relative positional encoding method for an attention mechanism that is robust to offset, it suffices to establish an angle matrix set that pairwise commutes.
Thus, in the following section, we focus on the construction of $\mathcal{A}$ to satisfy this particular requirement.
We call our method {\method} to indicate the \textbf{com}mutativity of angle matrices in \textbf{RoPE} function. The overview of {\method} is shown in Figure \ref{fig:overview}.


\begin{remark}
\tangyu{Among all the previous RoPE methods, LieRE \cite{ostmeier2024liere} is the most related one to ours which solves }
Equation $\ref{eq:exp-solution}$ by directly training the skew-symmetric matrix set $\mathcal A$.
However, it is worth noting that in the following equation of their implementation
\begin{equation}
\resizebox{.42\textwidth}{!}{$
(\mathbf R(\bm x_q; \mathcal A) \bm q)^\top(\mathbf R(\bm x_k; \mathcal A) \bm k)
=  \bm q^\top  \mathbf R(\bm x_q; \mathcal A)^\top  \mathbf R(\bm x_k; \mathcal A) \bm k
$},
\end{equation}
sometimes $\mathbf R(\bm x_q; \mathcal A)^\top  \mathbf R(\bm x_k; \mathcal A) \ne \mathbf R(\bm x_q - \bm x_k; \mathcal A) $ because the probability that two random matrices commute is small, where $\bm{x}_q$ and $\bm{x}_k$ are two arbitrary coordinates.
Thus, $\mathbf R$ proposed by LieRE does not consistently satisfy the requirements of the RoPE Equation.
\end{remark}

\subsection{Construction of pairwise commuting matrices}

In this section, we elaborate on concrete ways to construct the pairwise commuting matrices to solve the RoPE Equation parameterized by angle matrices.
Note that if two matrices are both block diagonal with the same block sizes, where the corresponding blocks are commutative, then these two matrices are commutative.
Formally speaking, for two matrices $\mathbf M, \mathbf N \in \mathbb R^{d\times d}$, they commute if
\begin{equation}
  \left\{
  \begin{array}{l}
    \vspace{1ex}
    \mathbf M = \mathrm{diag}(\mathbf M_1,\mathbf  M_2, \cdots,\mathbf  M_b) \\
    \vspace{5px}
    \mathbf N = \mathrm{diag}(\mathbf N_1, \mathbf N_2, \cdots,\mathbf  N_b) \\
    \vspace{5px}
    \mathbf M_i,\;\mathbf  N_i \in \mathbb R^{b \times b} \quad \quad \forall i \in \{1, 2, \cdots, m\} \\
    \mathbf M_i \mathbf N_i = \mathbf N_i \mathbf M_i \quad \quad \forall i \in \{1, 2, \cdots, m\} \\
    \end{array}
    \right.
\end{equation}
where $b$ denotes block size that is a factor of $d$ and $m = \frac{d}{b}$.

Thus, to present solutions to RoPE Equation parameterized by angle matrices, it suffices to partition the angle matrices $\mathbf A_i$ in \cref{eq:exp-solution} into $m$ blocks ${\mathbf B_{i1}, \mathbf B_{i2}, \cdots, \mathbf B_{im}}$ where:
\begin{equation}
\label{eq:definition-of-matrix-b}
\left\{
\begin{array}{l}
    \vspace{1ex}
    \mathbf A_i = \mathrm{diag}(\mathbf B_{i1}, \mathbf B_{i2}, \cdots, \mathbf B_{im}) \\
    \mathbf B_{ij} \in \mathbb R^{b \times b} \quad\quad \forall j \in \{1, 2, \cdots, m\} \\
\end{array}
\right.
\end{equation}

Defining $\mathcal B_j = \{\mathbf B_{1j}, \mathbf B_{2j}, \cdots, \mathbf B_{Nj}\}$, if $\mathcal B_j$ pairwise commutes for all $j \in \{1, 2, \cdots, m\}$, then $\mathcal A$ pairwise commutes.
Note that $\mathcal  B_1,\mathcal B_2, \cdots,\mathcal B_m$ are equivalent. Thus, without causing confusion, we will uniformly use $\mathcal B = \{\mathbf B_1, \mathbf B_2, \cdots,\mathbf B_N\}$ to represent $\mathcal B_j = \{ \mathbf B_{1j}, \mathbf B_{2j}, \cdots,\mathbf B_{Nj}\}$ in the following.

Currently, there is no general method that provides the necessary and sufficient conditions for ensuring that arbitrary trainable skew-symmetric matrices commute.
Therefore, we aim to establish sufficient conditions for enforcing this constraint. 
More concretely, we present Proposition \ref{prop:ap} and Proposition \ref{prop:ld} to construct two different variants of matrices, respectively.

\begin{proposition}
\label{prop:ap}
A set of matrices $\mathcal B = \{\mathbf B_1, \mathbf B_2, \cdots,\mathbf B_N\}$ pairwise commutes if all but one of them are zero matrices, i.e., $\exists\; k\in\{1,2,\cdots,N\}\;$ s.t.
\begin{equation}
  \mathbf   B_i = \mathbf O \; (\forall i \ne k)
\end{equation}
\end{proposition}

Based on Proposition \ref{prop:ap}, we propose Trainable RoPE parameterized by \textbf{Axial-Partition Angle Matrices (\method-AP)}. 
We divide the $m$ blocks into $N$ parts, such that each part is responsible for the rotation of a specific axis. 
In this case, $m$ should be a multiple of $N$.
Specifically, 
\begin{equation}
\label{eq:rope-ap}
\mathbf B_{ij} =
  \left\{
  \begin{array}{ll}
    \mathbf P_j - \mathbf P_j^\top, & \quad \text{if} \quad j\equiv i \pmod N \\
    \mathbf O, & \quad \text{otherwise}  \\
    \end{array}
    \right.
\end{equation}
where $\{\mathbf P_j\}_{j=1}^{m}$ represents a set of trainable matrices, $\mathbf O$ denotes a zero matrix, and $\equiv$ indicates congruence modulo.

\begin{proposition}
\label{prop:ld}
A set of matrices $\mathcal B = \{\mathbf B_1, \mathbf B_2, \cdots,\mathbf B_N\}$ pairwise commutes if they pairwise linearly dependent, i.e., $\exists\; \lambda_1,\lambda_2,\cdots,\lambda_N\in\mathbb R\;$ s.t.
\begin{equation}
    \lambda_1\mathbf  B_1 = \lambda_2 \mathbf B_2 = \cdots = \lambda_N\mathbf  B_N
\end{equation}
\end{proposition}

Based on Proposition \ref{prop:ld}, we propose Trainable RoPE parameterized by \textbf{Linearly-Dependent Angle Matrices (\method-LD)}. 
Specifically, 
We train a base matrix $\mathbf P$ with scaling factors $\{\theta_i\}_{i=1}^{N}$.
Then we obtain:
\begin{equation}
    \mathcal B = \{ \mathbf B_i = \theta_i (\mathbf P -\mathbf  P^\top)\; |\; i = 1, 2, \cdots, N \}
\end{equation}

\begin{figure}[t]
    \centering
    \vspace{-2em}
    \includegraphics[width=\linewidth]{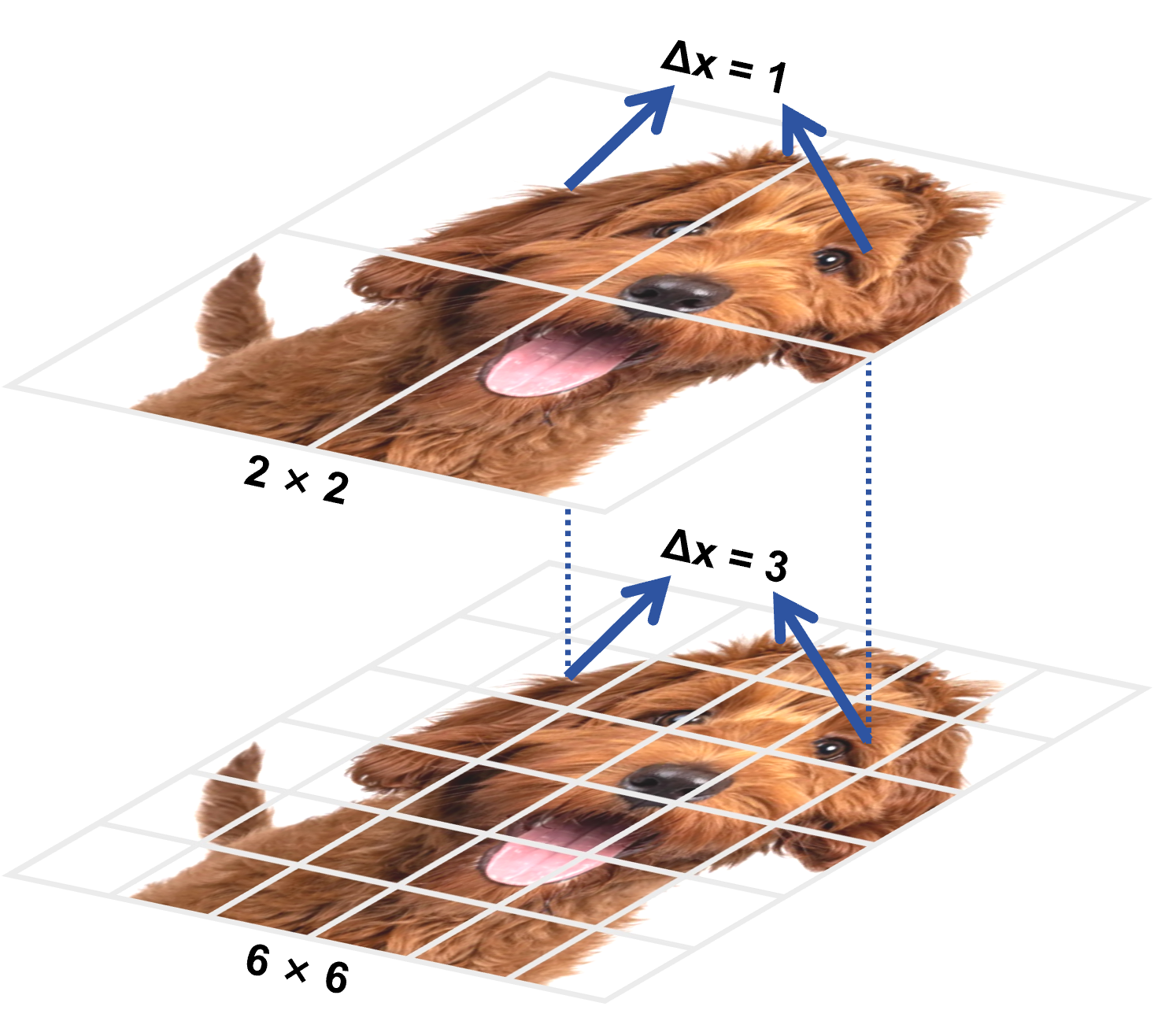}
    \vspace{-1em}
    \caption{Different patch sizes result in different relative relationships.}
    \label{fig:patch-span}
\end{figure}

\subsection{Implementation details of coordinates and improvements}

\subsubsection{Relative scaling and center offset} 

\noindent\textbf{Relative scaling}. In language models, positions are typically treated as discrete. Additionally, due to the relationship between tokens and the inherent uncertainty in sequence length, there is no need to scale these positions into a specified range. However, in the case of images, positions are continuous. When using different patch sizes, it becomes unreasonable to define their relative relationships solely based on the span of patches. As shown in Figure \ref{fig:patch-span}, the same locations with varying patch sizes can result in significantly different relative relationships.

As a result, applying relative coordinates is a better method for measuring relative relationships in images. For an image with shape $H\times W$, we scale both the height $H$ and the width $W$ to $1$. Therefore, for a pixel located at $(h, w)$ in the raw image, its coordinate is treated as $(\frac{h}{H}, \frac{w}{W})$. For high-dimensional coordinates, we perform the same operation, i.e., transform a raw coordinate $(x_1, x_2, \cdots, x_N)$ in multi-axial canvas with shape $(X_1, X_2, \cdots, X_N)$ into $(\frac{x_1}{X_1}, \frac{x_2}{X_2}, \cdots, \frac{x_N}{X_N})$.

\vspace{.7em}
\noindent\textbf{Center offset}. When we project a patch to a feature tensor with a coordinate, we aggregate all information of the patch into a specified location. 
Simply, we adopt the center point of a patch as the aggregation location.

\subsubsection{Position perturbation}

To achieve better robustness and excellent performance across different scales during inference, we add perturbations to the coordinates of the patches.
Specifically, for a patch with center $(x_1, \cdots, x_N)$ and size $(\Delta X_1, \cdots, \Delta X_N)$, during training, we formulate its location as:
\begin{equation}
    \mathcal N
    \left(
        \left(
                x_1, \cdots, x_N
        \right)^\top;
        \text{diag}(\sigma \Delta X_1 , \cdots, \sigma \Delta X_N )^2
    \right)
\end{equation}
where $\sigma$ is a hyper-parameter called perturbation intensity. Additionally, we truncate the location within the patch area, i.e., $ \left[ x_k - \frac{\Delta X_k}{2}, x_k + \frac{\Delta X_k}{2} \right] $.


\begin{table*}[t]
\centering
\begin{tabular}{c|c|ccccccccc}
\toprule
\multirow{2}{*}{\textbf{\makecell[c]{Position Encoding\\Method}}} & \multirow{2}{*}{\textbf{\makecell[c]{Perturbation\\Intensity}}} & \multicolumn{9}{c}{\textbf{Evaluation Resolution}}             \\

                                          &                                                  & \textbf{112}  & \textbf{128}   & \textbf{192}   & \textbf{224}   & \textbf{256}   & \textbf{320}   & \textbf{384}   & \textbf{448}   & \textbf{512}   \\ \midrule
\multirow{2}{*}{\textbf{APE}}             & 1                                                & 30.04         & 38.69          & 56.4           & 58.76          & 60.02          & 59.27          & 57.04          & 54.10           & 50.99          \\
                                          & 0                                                & 19.71         & 33.43          & 55.97          & 58.62          & 55.31          & 52.63          & 49.68          & 46.39          & 42.66          \\ \midrule
\multirow{2}{*}{\textbf{Vanilla RoPE}}    & 1                                                & 36.94         & 45.48          & 60.72          & 63.09          & 63.24          & 62.12          & 59.24          & 55.51          & 51.11          \\
                                          & 0                                                & 36.41         & 44.48          & 59.97          & 62.03          & 62.54          & 61.36          & 58.56          & 54.79          & 50.58          \\ \midrule
\multirow{2}{*}{\textbf{LieRE}} & 1                                                & 38.03         & {\ul 46.97}    & 62.22          & 64.36          & 64.99          & 63.78          & 61.15          & 57.92          & 53.74          \\
                                          & 0                                                & {\ul 38.22}   & 46.85          & 62.01          & 64.07          & 64.54          & 63.46          & 60.74          & 56.89          & 52.51          \\ \midrule
\multirow{2}{*}{\textbf{\method-AP (ours)}} & 1                                                & 35.75         & 46.18          & 62.82          & {\ul 65.32}    & {\ul 65.83}    & {\ul 64.78}    & {\ul 61.88}    & {\ul 58.21}    & {\ul 54.02}    \\
                                          & 0                                                & 37.17         & 46.95          & 62.63          & 64.76          & 65.11          & 64.35          & 61.62          & 58.10           & 53.81          \\ \midrule
\multirow{2}{*}{\textbf{\method-LD (ours)}} & 1                                                & \textbf{38.30} & \textbf{47.28} & \textbf{63.53} & \textbf{65.49} & \textbf{65.95} & \textbf{65.27} & \textbf{62.62} & \textbf{59.11} & \textbf{55.29} \\
                                          & 0                                                & 36.88         & 46.54          & {\ul 62.89}    & 65.27          & 65.66          & 64.27          & 61.83          & 57.87          & 53.64  \\ \bottomrule
\end{tabular}
\vspace{-.3em}
\caption{Accuracy of 2D classification on ImageNet. Models are trained at a resolution of $224\times 224$ and evaluated at varying resolutions.}
\label{tab:main-2d-cls}
\vspace{-.5em}
\end{table*}

\section{Experiments}

In this section, we evaluate the performance of various positional encoding methods on classic vision tasks.
We first assess their scalability in 2D image classification across different resolutions.
Additionally, we conduct object detection experiments to demonstrate the generalizability of our approach. 
To further examine the ability to handle higher-dimensional data, we perform 3D classification experiments, which are detailed in Appendix \switch{\ref{app:more-exp}}{B}.

\subsection{2D classification}

\subsubsection{Setup}

\noindent\textbf{Baselines and model architecture.}
We evaluate our proposed methods (\method-LD and \method-AP) against APE, vanilla RoPE (as introduced by RoFormer), and LieRE. To isolate the effects of positional encoding, we utilize a standard Vision Transformer (ViT-B/16) architecture with minimal modifications. The APE codebook is removed for methods that do not employ APE, and self-attention layers are replaced with RoPE self-attention parameterized by angle matrices. This design highlights the performance differences among various positional encoding methods. A block size of 8 is used in practice. More details are provided in Appendix \switch{\ref{app:reformulation}}{D}. 

\noindent\textbf{Training and evaluation protocol.}
All models are trained at a standard resolution of $224\times 224$ and evaluated across multiple resolutions to test their robustness and scalability on the ImageNet-1K dataset \citep{Deng2009ImageNetAL}.
The models are trained from scratch using randomly initialized parameters, ensuring no influence from pre-trained weights or external priors. To maintain fairness and reproducibility, we apply only basic data augmentation techniques, such as resizing and random cropping, focusing on relative performance comparisons rather than achieving absolute accuracy.
The primary evaluation metric is accuracy on the test set across various resolutions. Since APE is inherently fixed and discrete, bilinear interpolation is applied to adapt it to different resolutions during evaluation. 
Optimization methods and hyper-parameters are detailed in Appendix \switch{\ref{app:details}}{C}.

\begin{figure}[h]
    \centering
    \begin{subfigure}{1\linewidth}
        \centering
        \includegraphics[width=\linewidth]{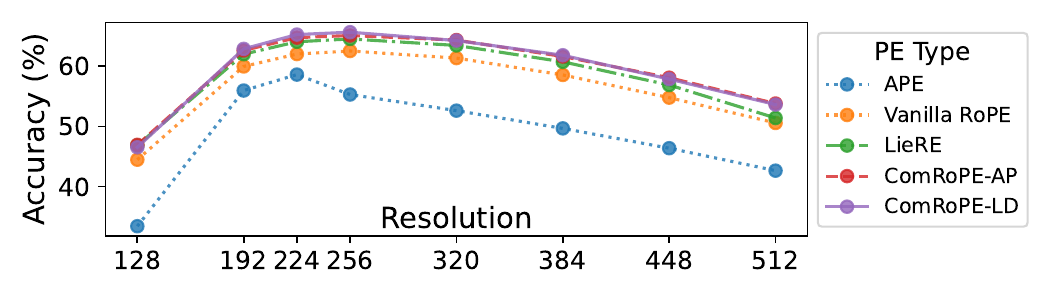}
        \vspace{-1.7em}
        \caption{Perturbation intensity = 0}
        \label{fig:2d-cls-pp0}
        \vspace{.3em}
    \end{subfigure}
    \hfill
    \begin{subfigure}{1\linewidth}
        \centering
        \includegraphics[width=\linewidth]{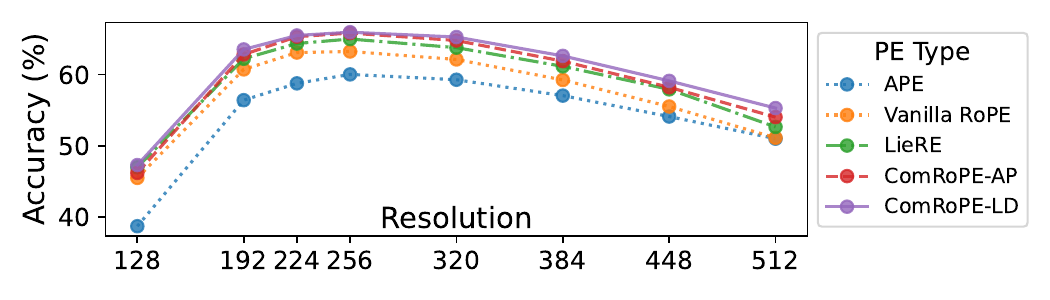}
        \vspace{-1.7em}
        \caption{Perturbation intensity = 1}
        \label{fig:2d-cls-pp1}
        \vspace{.3em}
    \end{subfigure}
    \vspace{-2em}
    \caption{Accuracy on ImageNet-1K for various positional encoding methods. The results for the same perturbation intensity are presented together for better comparison. For better visualization, evaluation resolution at $112\times 112$ is not included in these figures.}
    \label{fig:2d-cls-imagenet}
\end{figure}

\subsubsection{Main results}

The accuracy metrics for each positional encoding method across various resolutions are summarized in Table \ref{tab:main-2d-cls}. Additionally, Figure \ref{fig:2d-cls-imagenet} visually compares performance under different levels of perturbation intensity separately.

\noindent\textbf{Overall performance.}  
Across all evaluations, APE consistently exhibits the lowest accuracy, corroborating previous findings regarding its limitations in dynamic contexts. The vanilla RoPE shows a modest improvement over APE but remains less effective. In contrast, the trainable angle matrices, namely LieRE, \method-AP, and \method-LD, demonstrate significantly higher accuracy across all resolutions. Notably, \method-LD achieves the best performance among the three, suggesting that its inherent linear dependencies may enhance flexibility and structural learning capabilities.

\noindent\textbf{Accuracy at training resolution.} 
At the training resolution of $224\times 224$, all three methods with trainable angle matrices (\method-LD, \method-AP, and LieRE) achieve comparable accuracy, significantly outperforming both APE and the standard RoPE, which underscores the effectiveness of RoPE parameterized by trainable angle matrices. Notably, \method-LD surpasses the current state-of-the-art LieRE by $1.6\%$\footnote{Relative improvement is calculated as $\text{target} = \text{baseline} \times (1 + \text{improvement})$, where the improvement denotes the percentage increase over the baseline performance. This formula applies throughout the following sections.}.

\noindent\textbf{Scaling to higher resolution.} 
At resolutions beyond the training size, LieRE shows the steepest decline in accuracy among the three trainable RoPE variants, indicating greater sensitivity to resolution changes. In contrast, \method-LD and \method-AP exhibit a more gradual decrease in performance, thanks to their commutative properties that enhance positional robustness. Specifically, \method-LD outperforms LieRE by $2.9\%$ at a resolution of $512\times 512$.

\noindent\textbf{Significance of commutativity.}
These findings illustrate the effectiveness of trainable commutative angle matrices, particularly \method-LD, in maintaining accuracy and scalability across diverse resolutions. The results underscore the importance of commutativity in ensuring robust RoPE parameterized by trainable angle matrices for vision tasks. Furthermore, to better understand the role of commutativity, we conduct additional experiments by introducing coordinate offsets in LieRE (see Section \ref{sec:commutativity} for details).

\subsection{Object detection}

To demonstrate the generalizability of our approach, we conduct object detection experiments using the framework from \cite{xia2024vitcomer}. We adopt ViT-S as our backbone and apply {\method} to the attention layers. To ensure consistency with the pre-trained model, we initialize the angle matrix to zero.

We evaluate \method-LD, LieRE, and APE on the MS COCO dataset \citep{lin2014microsoft}. As summarized in Table \ref{tab:object-detection}, both \method-LD and LieRE outperform APE, with \method-LD achieving slightly better performance than LieRE while only using nearly half the number of extra parameters.

\begin{table}[h]
\centering
\resizebox{\linewidth}{!}{
    \begin{tabular}{ccccccc}
    \toprule
    \textbf{PE Method}  & \textbf{AP}   & \textbf{AP$^{50}$} & \textbf{AP$^{75}$} & \textbf{AP$^\text{s}$}  & \textbf{AP$^\text{m}$}  & \textbf{AP$^\text{l}$}  \\ \midrule
    \textbf{APE}        & 44.0          & 66.6          & 47.7          & 28.2          & 46.8          & 58.4          \\
    \textbf{LieRE}      & 44.5          & 67.3          & 48.4          & 29.0          & 46.9          & 58.7          \\
    \textbf{ComRoPE-LD} & \textbf{44.7} & \textbf{67.6} & \textbf{48.5} & \textbf{29.2} & \textbf{47.1} & \textbf{60.0} \\ \bottomrule
    \end{tabular}
}
\vspace{-.7em}
\caption{Results of object detection on MS COCO.}
\label{tab:object-detection}
\end{table}

To compare training efficiency, we plot the results for each epoch in Figure \ref{fig:detection-epochs}. Our findings indicate that both \method-LD and LieRE converge faster than APE, requiring $29\%$ fewer iterations to achieve the same results as APE.

\begin{figure}[h]
    \centering
    \vspace{-1em}
    \includegraphics[width=1\linewidth]{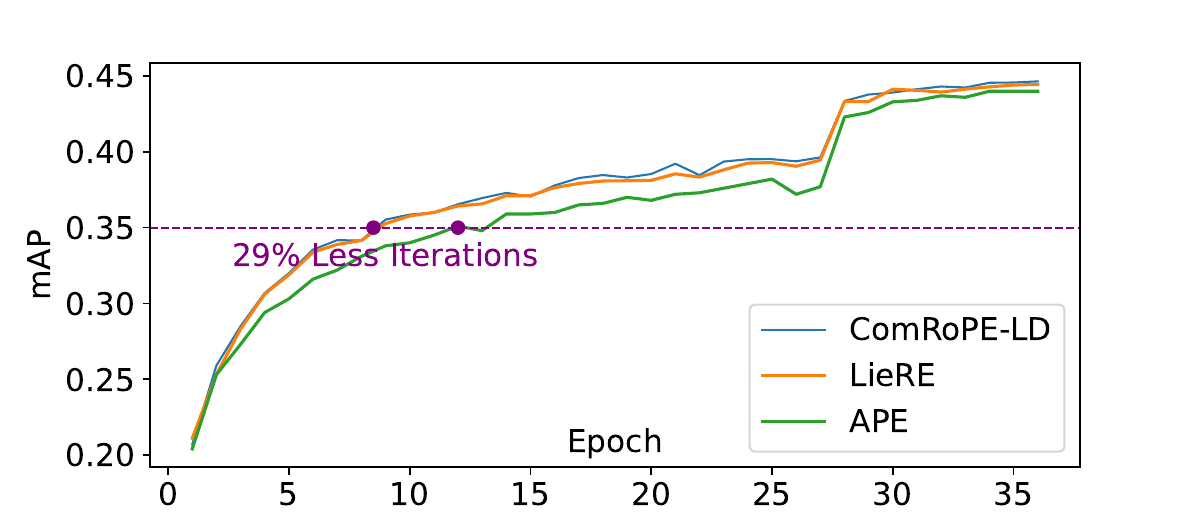}
    \vspace{-2em}
    \caption{Results over the whole training procedure.}
    \label{fig:detection-epochs}
    \vspace{-1em}
\end{figure}




\subsection{Ablation study}

\subsubsection{Impact of commutativity}
\label{sec:commutativity}
To evaluate the significance of the commutativity of angle matrices, we conduct experiments on the LieRE by introducing a coordinate offset. Specifically, before inference, a random offset is applied uniformly across all coordinates within the image. The offset is sampled from a Gaussian distribution as follows:
\begin{equation}
    \mathcal{N} \left( 0; \; \rho^2 \cdot \mathbf{I}_{N \times N} \right)
\end{equation}
where $\rho$ represents the standard deviation of the random offset, and $\mathbf{I}_{N \times N}$ is the identity matrix of size $N \times N$. It is important to note that applying the same offset to all coordinates does not influence the relative positional dependencies between the patches. The other experimental settings remain consistent with those in the 2D classification tasks.

The results shown in Figure \ref{fig:abl-offset} demonstrate that the baseline model's performance deteriorates significantly as the standard deviation of the offset increases. In contrast, our proposed model, {\method}, maintains consistent performance across all offset values. This is due to the commutativity of the angle matrices, which allows the model to remain invariant to such coordinate shifts. The robustness of {\method} to this type of perturbation highlights its capacity to preserve relative positional information, even in the presence of modification introduced by coordinate offsets.

\begin{figure}[h]
    \centering
    \includegraphics[width=1\linewidth]{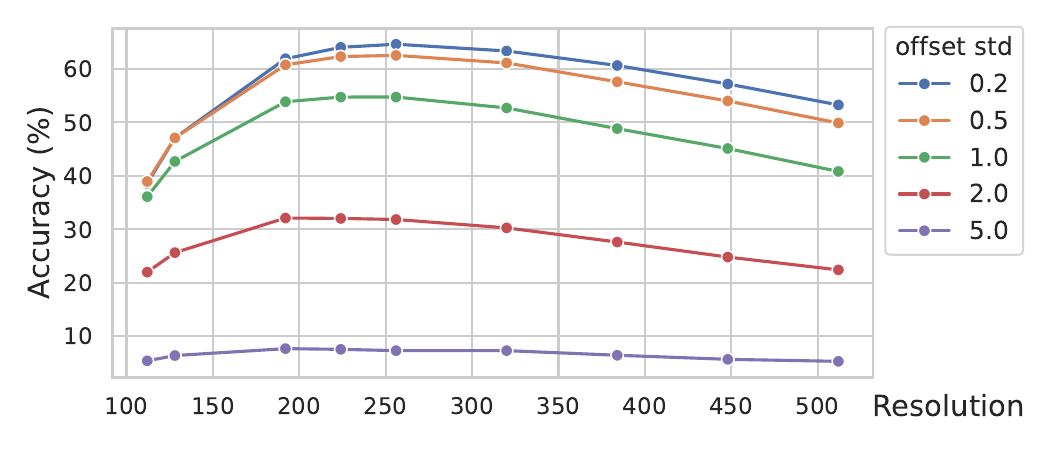}
    \vspace{-2em}
    \caption{Effect of coordinate offset on LieRE. As the standard deviation of the offset increases, the performance of the baseline model deteriorates, while {\method} remains unaffected (unpainted).}
    \label{fig:abl-offset}
    \vspace{-.5em}
\end{figure}





\subsubsection{Impact of block size}


In this section, we examine the impact of block size on 2D classification using the ImageNet dataset. We maintain the same experimental setup as in our primary experiments, varying the block size from 2 to 8. The results are presented in Figure \ref{fig:ablation-block-size}. 
\begin{figure}[h]
    \centering
    \begin{subfigure}{1\linewidth}
        \centering
        \includegraphics[width=\linewidth]{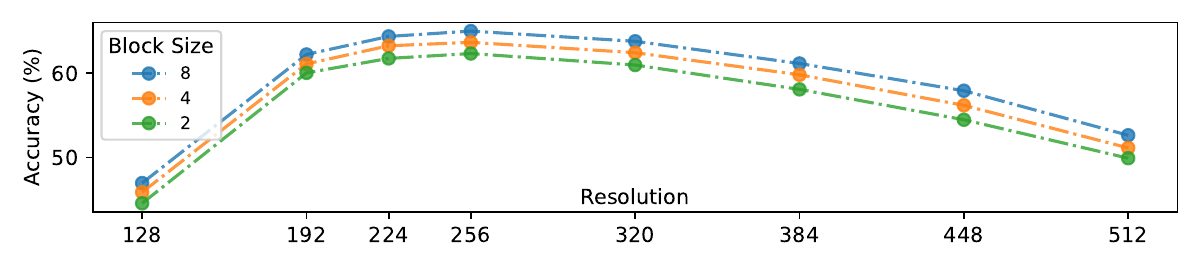}
        \vspace{-1.7em}
        \caption{LieRE}
        \vspace{.3em}
    \end{subfigure}
    \hfill
    \begin{subfigure}{1\linewidth}
        \centering
        \includegraphics[width=\linewidth]{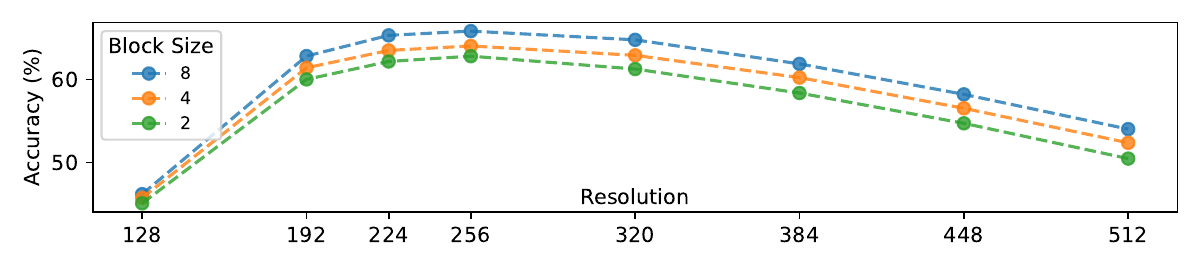}
        \vspace{-1.7em}
        \caption{\method-AP}
        \vspace{.3em}
    \end{subfigure}
    \hfill
    \begin{subfigure}{1\linewidth}
        \centering
        \includegraphics[width=\linewidth]{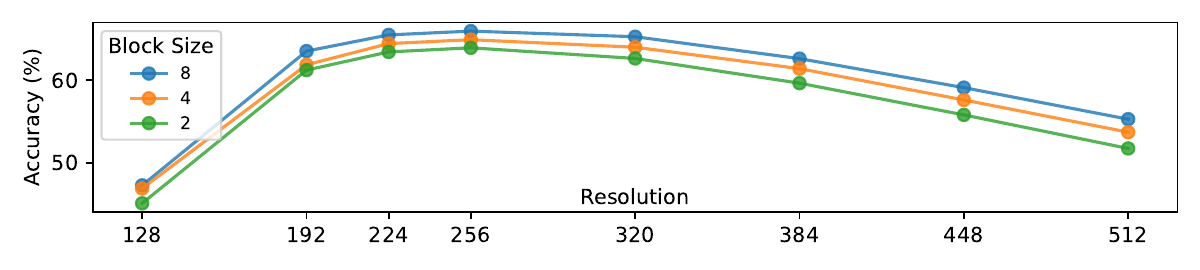}
        \vspace{-1.7em}
        \caption{\method-LD}
    \end{subfigure}
    \vspace{-2em}
    \caption{Accuracy on ImageNet for various block sizes. Larger block size results in better performance.}
    \label{fig:ablation-block-size}
    \vspace{-1em}
\end{figure}
Our findings indicate that larger block sizes consistently improve performance by extending the rotation transformation space to a more significant subgroup of the particular orthogonal group by introducing additional parameters and computation time. When the block size is small, the associated costs are minimal. However, as the block size increases, the primary term of time complexity grows to \(O(Lndb^2)\) from \(O(Lnd\cdot \frac{d}{h})\), which becomes significant. Therefore, we limit the block size to a maximum of 8 to balance performance with additional costs. In other words, we select a block size of 8 to optimize performance while keeping the extra computational cost manageable.

\subsubsection{Utility of position perturbation}

In this section, we explore the impact of positional perturbations. 
We conducted experiments on \method-LD and APE using the ImageNet dataset, with the results presented in Figure \ref{fig:acc-over-eval-res}. 

\begin{figure}[h]
    \centering
    \begin{subfigure}{1\linewidth}
        \centering
        \includegraphics[width=\linewidth]{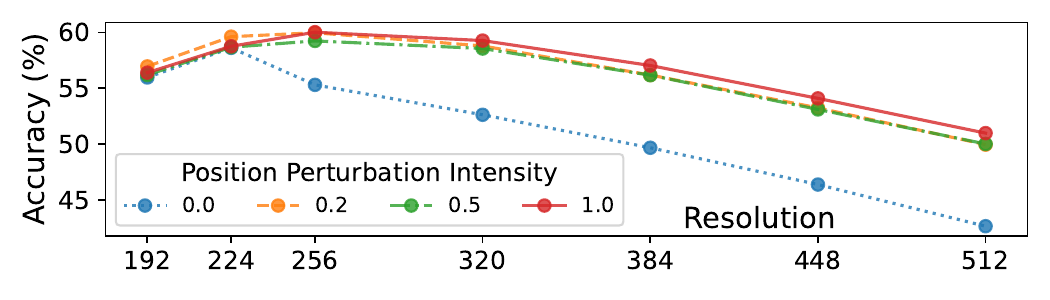}
        \vspace{-1.7em}
        \caption{APE}
        \vspace{.3em}
        \label{fig:acc-over-eval-res-ape}
    \end{subfigure}
    \hfill
    \begin{subfigure}{1\linewidth}
        \centering
        \includegraphics[width=\linewidth]{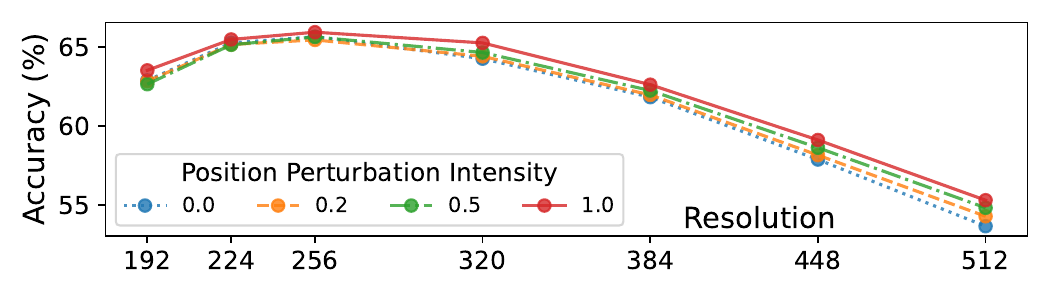}
        \vspace{-1.7em}
        \caption{ComRoPE-LD}
        \label{fig:acc-over-eval-res-ld}
    \end{subfigure}
    \vspace{-2em}
    \caption{Accuracy on ImageNet over varying position perturbation intensity.}
    \label{fig:acc-over-eval-res}
    \vspace{-1em}
\end{figure}

As shown in Figure \ref{fig:acc-over-eval-res-ape}, APE is highly sensitive to positional perturbations, leading to significant performance improvement ($+19.5\%$ when increasing intensity from 0 to 1) when these perturbations are introduced.
For RoPE with angle matrices shown in Figure \ref{fig:acc-over-eval-res-ld}, positional perturbations also resulted in some performance gains ($+2.9\%$), though the improvement was relatively modest.
This is likely due to the inherent robustness of the RoPE design with angle matrices, which is already well-equipped to handle variations in position.

\subsection{Applications}

In our approach, when the angle matrix is an all-zero matrix, the rotation matrix becomes the identity matrix, causing RoPE Attention to reduce to the standard Attention mechanism. When the block size of the angle matrix is set to 2, \method-AP effectively reduces to the commonly used RoPE Attention in language models. This demonstrates that our method can represent standard Attention and various common variants of RoPE Attention. Therefore, during the fine-tuning stage, we can replace standard Attention with our method, load the pre-trained weights, and fine-tune them under the new paradigm. In other words, \method{} can be seamlessly integrated into the fine-tuning process, even if it was not applied during pre-training. Additional experiments can be found in Appendix \switch{\ref{app:more-exp}}{B}.

\section{Conclusion}
\label{sec:conclusion}

In this work, we proposed {\method}, a novel and adaptive framework for Rotary Position Embedding (RoPE) parameterized by trainable angle matrices. 
We rigorously formulate the RoPE Equation and then establish a necessary and sufficient condition for its solution.
Our approach effectively overcomes the scalability and robustness limitations of existing RoPE methods by eliminating the need for manually designed rotation matrices and introducing a more flexible, scalable solution. Extensive experimental results show that {\method} outperforms the existing positional encoding methods across various tasks. 
Furthermore, our framework generalizes existing RoPE formulations and demonstrates the potential for broader application in Transformer models, offering insights and a solid foundation for future research in positional encoding techniques.

\clearpage

\section*{Acknowledgment}

This work is supported by the National Key R\&D Program of China (2022YFB4701400/4701402), SSTIC Grant (KJZD20230923115106012, KJZD20230923114916032, GJHZ20240218113604008) and Beijing Key Lab of Networked Multimedia.

{
    \small
    \bibliographystyle{ieeenat_fullname}
    \bibliography{main}
}


\clearpage
{
    \setcounter{page}{1}
    \maketitlesupplementary
    \appendix
    
\begin{table*}[t]
\centering
\begin{tabular}{c|c|ccccccccc}
\toprule
\multirow{2}{*}{\textbf{\makecell[c]{Position Encoding\\Method}}} & \multirow{2}{*}{\textbf{\makecell[c]{Perturbation\\Intensity}}} & \multicolumn{9}{c}{\textbf{Evaluation Resolution}}             \\

                                          &                                                  & \textbf{112}  & \textbf{128}   & \textbf{192}   & \textbf{224}   & \textbf{256}   & \textbf{320}   & \textbf{384}   & \textbf{448}   & \textbf{512}   \\ \midrule
\multirow{2}{*}{\textbf{APE}}          & 1                                                & 48.10          & 55.25          & 76.50          & 93.10          & 76.70          & 71.48          & 74.36          & 62.23          & 53.18          \\
                                       & 0                                                & 45.54          & 55.18          & 76.48          & 92.87          & 75.91          & 70.70          & 73.70          & 60.43          & 48.79          \\ \midrule
\multirow{2}{*}{\textbf{Vanilla RoPE}} & 1                                                & 47.28          & 54.96          & 75.69          & 93.79          & 77.66          & 72.53          & 74.72          & 65.19          & 57.34          \\
                                       & 0                                                & 48.12          & 55.21          & 76.47          & 94.12          & 76.72          & 71.59          & 74.47          & 62.28          & 53.99          \\ \midrule
\multirow{2}{*}{\textbf{LieRE}}        & 1                                                & 48.97          & 56.15          & 77.33          & {\ul 94.43}    & 78.35          & 73.20          & {\ul 77.33}    & 65.74          & 58.23          \\
                                       & 0                                                & 48.75          & 55.46          & 78.16          & 94.24          & 78.91          & 72.92          & 76.86          & 65.35          & 56.85          \\ \midrule
\multirow{2}{*}{\textbf{ComRoPE-AP}}   & 1                                                & \textbf{50.14} & 55.63          & 77.47          & 94.37          & 79.27          & 73.56          & 76.66          & \textbf{67.68} & {\ul 59.34}    \\
                                       & 0                                                & 48.06          & 55.63          & 75.72          & 94.26          & 75.75          & 70.93          & 74.72          & 64.91          & 57.98          \\ \midrule
\multirow{2}{*}{\textbf{ComRoPE-LD}}   & 1                                                & {\ul 49.89}    & \textbf{56.60} & \textbf{79.21} & 94.24          & \textbf{80.27} & \textbf{74.22} & \textbf{78.60} & {\ul 67.46}    & \textbf{60.39} \\
                                       & 0                                                & 48.70          & {\ul 56.46}    & {\ul 78.30}    & \textbf{94.48} & {\ul 79.27}    & {\ul 74.58}    & 76.86          & 66.02          & 57.68      \\ \bottomrule
\end{tabular}
\caption{Accuracy of 3D classification on UCF-101. Models are trained at a resolution of $224\times 224$ and evaluated at varying resolutions.}
\label{tab:main-3d-cls}
\end{table*}

\section{Theorems and proofs}
\label{app:theorems}

\subsection{Proof of the main theorem}



To prove our main theorem (\ie, Theorem \ref{theo:main}), we first propose some lemmas and prove them.

\begin{lemma}
    \label{the:two_commute}
    Matrices $\mathbf{A}, \mathbf{B} \in \mathbb{R}^{n \times n}$ commute if and only if $e^{\mathbf{A}x} e^{\mathbf{B}y} = e^{\mathbf{A}x + \mathbf{B}y}$ for all $x, y \in \mathbb{R}$.
\end{lemma}

\begin{proof}
\quad

\noindent\textbf{1) Necessity ($\Rightarrow$).} 
By the definition of $e^{\mathbf{A}}$, we have:
\begin{equation}
\begin{split}
     e^{\mathbf{A} + \mathbf{B}} & = \sum_{n=0}^{\infty} \frac{(\mathbf{A} + \mathbf{B})^n}{n!} \\
     & = \sum_{n=0}^{\infty} \frac{\sum_{k=0}^{n} \binom{n}{k} \mathbf{A}^k \mathbf{B}^{n-k}}{n!} \\
     & = \sum_{n=0}^{\infty} \sum_{k=0}^{n} \frac{\mathbf{A}^k \mathbf{B}^{n-k}}{k!(n-k)!} \\
     & = \left(\sum_{k=0}^{\infty} \frac{\mathbf{A}^k}{k!}\right) \left(\sum_{m=0}^{\infty} \frac{\mathbf{B}^m}{m!}\right) \\
     & = e^{\mathbf{A}} e^{\mathbf{B}}.
\end{split}
\end{equation}
Substituting $\mathbf{A}, \mathbf{B}$ with $\mathbf{A}x, \mathbf{B}y$, we obtain:
\begin{equation}
     e^{\mathbf{A}x + \mathbf{B}y} = e^{\mathbf{A}x} e^{\mathbf{B}y}.
\end{equation}

\noindent\textbf{2) Sufficiency ($\Leftarrow$).}
We have:
\begin{equation}
\begin{split}
    e^{\mathbf{A}t} e^{\mathbf{B}t} & = \left(\sum_{n=0}^{\infty} \frac{t^n \mathbf{A}^n}{n!}\right) \left(\sum_{m=0}^{\infty} \frac{t^m \mathbf{B}^m}{m!}\right) \\
    & = \mathbf{I} + t(\mathbf{A} + \mathbf{B}) + t^2 \cdot \frac{\mathbf{A}^2 + 2\mathbf{A}\mathbf{B} + \mathbf{B}^2}{4} + o(t^2),
\end{split}
\end{equation}
and
\begin{equation}
\begin{split}
    e^{(\mathbf{A}+\mathbf{B})t} & = \sum_{n=0}^{\infty} \frac{((\mathbf{A} + \mathbf{B})t)^n}{n!} \\
    & = \mathbf{I} + t(\mathbf{A} + \mathbf{B}) + t^2 \cdot \frac{(\mathbf{A} + \mathbf{B})^2}{4} + o(t^2).
\end{split}
\end{equation}
Let $t^2 f(t)$ be the difference between the two expressions above. Thus, we obtain:
\begin{equation}
\begin{split}
    f(t) & = \frac{e^{\mathbf{A}t} e^{\mathbf{B}t} - e^{(\mathbf{A} + \mathbf{B})t}}{t^2} \\
    & = \frac{\mathbf{A}^2 + 2\mathbf{A}\mathbf{B} + \mathbf{B}^2}{4} - \frac{(\mathbf{A} + \mathbf{B})^2}{4} + o(1) \\
    & = \frac{\mathbf{A}\mathbf{B} - \mathbf{B}\mathbf{A}}{4} + o(1).
\end{split}
\end{equation}
Taking the limit as $t \rightarrow 0$, we have:
\begin{equation}
\lim_{t \rightarrow 0} f(t) = \frac{\mathbf{A}\mathbf{B} - \mathbf{B}\mathbf{A}}{4}.
\end{equation}
Since $e^{\mathbf{A}x} e^{\mathbf{B}y} = e^{\mathbf{A}x + \mathbf{B}y}$, we have $f(t) = 0$, which implies $\mathbf{A}\mathbf{B} = \mathbf{B}\mathbf{A}$.

\end{proof}

\begin{lemma}
\label{the:multi_commute}
Matrices $\mathbf{A}_1, \mathbf{A}_2, \ldots, \mathbf{A}_m \in \mathbb{R}^{n \times n} \; (m > 1)$ pairwise commute if and only if:
\begin{equation}
    e^{\mathbf{A}_1 x_1} e^{\mathbf{A}_2 x_2} \cdots e^{\mathbf{A}_m x_m} = e^{\mathbf{A}_1 x_1 + \mathbf{A}_2 x_2 + \cdots + \mathbf{A}_m x_m}
\end{equation}
for all $x_1, x_2, \ldots, x_m \in \mathbb{R}$.
\end{lemma}

\begin{proof}
For $m = 2$, the theorem holds by Lemma \ref{the:two_commute}. Suppose the theorem holds for all $2 \leq m \leq k$. We prove it for $m = k+1$.

\noindent\textbf{1) Necessity ($\Rightarrow$).} 
Assuming:
\begin{equation}
    e^{\mathbf{A}_1 x_1} e^{\mathbf{A}_2 x_2} \cdots e^{\mathbf{A}_k x_k} = e^{\mathbf{A}_1 x_1 + \mathbf{A}_2 x_2 + \cdots + \mathbf{A}_k x_k},
\end{equation}
we split $\mathbf{A}_1 x_1 + \mathbf{A}_2 x_2 + \cdots + \mathbf{A}_{k+1} x_{k+1}$ into two parts:
\begin{equation}
\begin{split}
    & \mathbf{A}_1 x_1 + \mathbf{A}_2 x_2 + \cdots + \mathbf{A}_{k+1} x_{k+1} \\
    & = (\mathbf{A}_1 x_1 + \mathbf{A}_2 x_2 + \cdots + \mathbf{A}_k x_k) + (\mathbf{A}_{k+1} x_{k+1}).
\end{split}
\end{equation}
Since $\mathbf{A}_1, \mathbf{A}_2, \ldots, \mathbf{A}_{k+1}$ commute in pairs, $\mathbf{A}_1 x_1 + \mathbf{A}_2 x_2 + \cdots + \mathbf{A}_k x_k$ and $\mathbf{A}_{k+1} x_{k+1}$ also commute. Thus:
\begin{equation}
\begin{split}
    &e^{\mathbf{A}_1 x_1 + \mathbf{A}_2 x_2 + \cdots + \mathbf{A}_{k+1} x_{k+1}} \\
    &= e^{(\mathbf{A}_1 x_1 + \mathbf{A}_2 x_2 + \cdots + \mathbf{A}_k x_k)} e^{\mathbf{A}_{k+1} x_{k+1}} \\
    &= e^{\mathbf{A}_1 x_1} e^{\mathbf{A}_2 x_2} \cdots e^{\mathbf{A}_{k+1} x_{k+1}}.
\end{split}
\end{equation}

\noindent\textbf{2) Sufficiency ($\Leftarrow$).} 
Let $x_{k+1} = 0$. Then:
\begin{equation}
    e^{\mathbf{A}_1 x_1} e^{\mathbf{A}_2 x_2} \cdots e^{\mathbf{A}_k x_k} = e^{\mathbf{A}_1 x_1 + \mathbf{A}_2 x_2 + \cdots + \mathbf{A}_k x_k},
\end{equation}
implying that $\mathbf{A}_1, \mathbf{A}_2, \ldots, \mathbf{A}_k$ commute in pairs.

For any $p \in \{1, 2, \ldots, k\}$, set all $x_i = 0$ except for $x_p$ and $x_{k+1}$. This yields:
\begin{equation}
    e^{\mathbf{A}_p x_p} e^{\mathbf{A}_{k+1} x_{k+1}} = e^{\mathbf{A}_p x_p + \mathbf{A}_{k+1} x_{k+1}},
\end{equation}
which implies $\mathbf{A}_p$ and $\mathbf{A}_{k+1}$ commute. Thus, $\mathbf{A}_1, \mathbf{A}_2, \ldots, \mathbf{A}_{k+1}$ commute in pairs.

\end{proof}

\begin{lemma}
\label{the:f_exists}
Matrices $\mathbf{A}_1, \mathbf{A}_2, \ldots, \mathbf{A}_m \in \mathbb{R}^{n \times n} \; (m > 1)$ pairwise commute if and only if 
there exists a function $f: \mathbb{R}^m \rightarrow \mathbb{R}^{n \times n}$ such that:
\begin{equation}
\label{eq:f_eqs_exey}
\begin{split}
    &  f(x_1 + y_1, x_2 + y_2, \ldots, x_m + y_m)  \\
    & =  e^{\mathbf{A}_1 x_1 + \mathbf{A}_2 x_2 + \cdots + \mathbf{A}_m x_m} e^{\mathbf{A}_1 y_1 + \mathbf{A}_2 y_2 + \cdots + \mathbf{A}_m y_m}
\end{split}
\end{equation}
for all $x_1, y_1, x_2, y_2, \ldots, x_m, y_m \in \mathbb{R}$.
\end{lemma}

\begin{proof}
\noindent\textbf{1) Necessity ($\Rightarrow$).} 
By Lemma \ref{the:multi_commute}, we can easily verify that the following $f$ satisfies the condition:
\begin{equation}
\begin{split}
    & f(x_1 + y_1, x_2 + y_2, \ldots, x_m + y_m) \\
    & = e^{\mathbf{A}_1 (x_1 + y_1) + \mathbf{A}_2 (x_2 + y_2) + \cdots + \mathbf{A}_m (x_m + y_m)}.
\end{split}
\end{equation}

\noindent\textbf{2) Sufficiency ($\Leftarrow$).}
From Equation \ref{eq:f_eqs_exey}, let $x_k$ be replaced with $x_k + y_k$ and $y_k$ with $0$. We obtain:
\begin{equation}
\begin{split}
    & f(x_1 + y_1, x_2 + y_2, \ldots, x_m + y_m) \\
    & = e^{\mathbf{A}_1 (x_1 + y_1) + \cdots + \mathbf{A}_m (x_m + y_m)} e^{\mathbf{A}_1 \cdot 0 + \cdots + \mathbf{A}_m \cdot 0} \\
    & = e^{\mathbf{A}_1 (x_1 + y_1) + \cdots + \mathbf{A}_m (x_m + y_m)}.
\end{split}
\end{equation}
Comparing this with Equation \eqref{eq:f_eqs_exey}, we get:
\begin{equation}
\begin{split}
    & e^{\mathbf{A}_1 x_1 + \mathbf{A}_2 x_2 + \cdots + \mathbf{A}_m x_m} e^{\mathbf{A}_1 y_1 + \mathbf{A}_2 y_2 + \cdots + \mathbf{A}_m y_m}  \\
    &= e^{\mathbf{A}_1 (x_1 + y_1) + \cdots + \mathbf{A}_m (x_m + y_m)}.
\end{split}
\end{equation}
For any $i, j \in \{1, 2, \ldots, m\}$, set $x_{k} = 0$ for all $k \neq i$ and $y_{k} = 0$ for all $k \neq j$. This leads to:
\begin{equation}
    e^{\mathbf{A}_i x_i} e^{\mathbf{A}_j y_j} = e^{\mathbf{A}_i x_i + \mathbf{A}_j y_j}.
\end{equation}
By Lemma \ref{the:two_commute}, this implies that $\mathbf{A}_i$ and $\mathbf{A}_j$ commute.
Therefore, matrices $\mathbf{A}_1, \mathbf{A}_2, \ldots, \mathbf{A}_m \in \mathbb{R}^{n \times n} \; (m > 1)$ pairwise commute.

\end{proof}


\begin{proof}[Proof of Theorem \ref{theo:main}]
Recall that $e^{\mathbf{A}}$ is an orthogonal matrix if $\mathbf{A}$ is skew-symmetric, which implies $\mathbf R(\bm{x};\mathcal A)^\top = \mathbf R(\bm{x};\mathcal A)^{-1} = \mathbf R(-\bm{x};\mathcal A)$.
Thus, we have:
\begin{equation}
\label{eq:prf-of-theo}
\begin{split}
    & \mathbf R(\bm x; \mathcal{A})^\top\mathbf R(\bm y; \mathcal{A}) \\
    &= e^{-\mathbf{A}_1 x_1 - \mathbf{A}_2 x_2 - \cdots - \mathbf{A}_N x_N} 
    e^{\mathbf{A}_1 y_1 + \mathbf{A}_2 y_2 + \cdots + \mathbf{A}_N y_N}.
\end{split}
\end{equation}
By Lemma \ref{the:f_exists} and Equation \ref{eq:prf-of-theo}, 
$\mathcal{A}$ pairwise commute if and only if there exists a function $f: \mathbb{R}^N \rightarrow \mathbb{R}^{d \times d}$ such that:
\begin{equation}
\label{eq:f_eqs_exey}
\begin{split}
    &  f(y_1 - x_1, y_2 - x_2, \ldots, y_N - x_N)  \\
    & =  \mathbf R(\bm x; \mathcal{A})^\top\mathbf R(\bm y; \mathcal{A}).
\end{split}
\end{equation}
Therefore, the theorem holds.

\end{proof}

\subsection{Explanation of rotation matrix and its exponential representation}

Following the definition in \cite{Grove2001ClassicalGA}, we first demonstrate the definition of rotation group and rotation matrix:

\begin{definition}[Rotation Group and Rotation Matrix]
A \textbf{special orthogonal group} in $\mathbb{R}^n$, denoted $SO(n)$, is the set of all $n \times n$ orthogonal matrices with determinant 1, \ie,
\[
SO(n) = \{ \mathbf R \in \mathbb{R}^{n \times n} \mid \mathbf  R^\top \mathbf R = \mathbf I, \det(\mathbf R) = 1 \}.
\]
We use the terms \textbf{rotation group} and \textbf{special orthogonal group} interchangeably. Any matrix in the rotation group is called a \textbf{rotation matrix}.
\end{definition}

To establish Definition \ref{def:extend_rotation_matrix}, there is a necessary proposition to ensure the correctness of the exponential representation of a rotation matrix:

\begin{proposition}
\label{prop:lie-algebra}
    Any rotation matrix $\mathbf R$ can be represented by $\exp(\mathbf A)$ where $\mathbf A$ is a skew-symmetric matrix.
\end{proposition}

Proposition \ref{prop:lie-algebra} is a well-known result in Lie theory, as detailed in \cite{Gallier2001BasicsOC}. Specifically, the matrix $\mathbf{R}$ in Proposition \ref{prop:lie-algebra} belongs to the Lie group $SO(n)$. The associated Lie algebra of this group is $\mathfrak{so}(n)$, within which the skew-symmetric matrix $\mathbf{A}$ resides.

    \section{More experiments}

\label{app:more-exp}

\subsection{3D classification}

To assess the ability to handle higher dimensions beyond 2D, we conduct a 3D classification task on UCF-101 \citep{soomro2012ucf101}. 
The details of the model and configuration can be found in Appendix \ref{app:details}.

The results shown in Table \ref{tab:main-3d-cls} demonstrate similar results in 2D experiments, that {\method} performs best when resolution increases beyond the training resolution, displaying the resolution robustness of {\method}.

\subsection{Fine-tune on pre-trained model}



Recall that we represent the RoPE function parameterized by angle matrices as defined in Equation~\ref{eq:exp-solution}. If all elements in $\mathcal{A} = \{\mathbf{A}_i\}_{i=1}^{N}$ are initialized as zero matrices (i.e., $\forall i, \mathbf{A}_i = \mathbf{O}$), the behavior of this RoPE function degenerates into a standard attention mechanism. This is because, in this case, $\mathbf{R}(\bm{x}; \mathcal{A}) = \exp(\mathbf{O}) = \mathbf{I}$ for any input $\bm{x}$. 

On the other hand, if $\mathcal{A} = \{\mathbf{A}_i\}_{i=1}^{N}$ is initialized as described in Appendix~\ref{app:reformulation}, the RoPE function reduces to the vanilla RoPE formulation.

These observations demonstrate that our method can represent both the standard attention mechanism and various common RoPE attention variants. Therefore, during fine-tuning, standard attention or vanilla RoPE can be replaced with our method. Pre-trained weights can be loaded and fine-tuned under this new paradigm seamlessly, even if \method{} was not applied during the pre-training phase.

As an example, we fine-tune the Vision Transformer pre-trained in CLIP \citep{radford2021learning} on ImageNet by simply replacing the standard attention mechanism with each RoPE method. Specifically, we fine-tune the model for 4 epochs using a batch size of 3456 and a learning rate of $3 \times 10^{-4}$. 

The results, presented in Table~\ref{tab:finetune-on-imagenet}, show that \method-LD achieves the best performance. An interesting observation is that vanilla RoPE exhibits the lowest accuracy among all five methods. This is likely because its fixed and manually defined parameters cannot be loaded seamlessly. In other words, it must adapt the pre-trained latent space during fine-tuning to effectively complete the task, which may result in suboptimal performance.

\begin{table}[h]
    \centering
    \begin{tabular}{c|c}
        \toprule
        Method & Accuracy \\ \midrule
        APE & 79.91 \\
        Vanilla RoPE & 79.82 \\
        LieRE & {\ul 80.12} \\
        \method-AP (ours) & 80.11 \\
        \method-LD (ours) & \textbf{80.17} \\ \bottomrule
    \end{tabular}
    \caption{Accuracy of fine-tuned models with different positional encoding methods on ImageNet.}
    \label{tab:finetune-on-imagenet}
\end{table}
    \section{Details of configuration}
\label{app:details}

\subsection{Configuration of 2D classification}

Configuration of 2D classfication task is shown in Table \ref{tab:config-2d-cls}.

\begin{table}[h]
    \centering
    \begin{tabular}{c|c}
        \toprule
        \textbf{Key} & \textbf{Value} \\
        \midrule
        Layers & 12 \\
        Image Size & 224 \\
        Patch Size & 16 \\
        Hidden Dimension & 768 \\
        Attention Heads & 12 \\
        Batch Size & 6144 \\
        Optimizer & AdamW \\
        Weight Decay & 0.01 \\
        Learning Rate & $10^{-3}$ \\
        LR Scheduler & cosine \\
        Warmup Ratio & 0.02 \\
        Epochs & 200 \\
        \bottomrule
    \end{tabular}
    \caption{Model and training configuration of 2D classification experiments.}
    \label{tab:config-2d-cls}
\end{table}

\subsection{Configuration of 3D classification}

Because the vanilla RoPE and \method-AP require that the head dimension be a multiple coordinate dimension, standard ViT-Base is not applicable. 
We modified the model parameters to make it possible to conduct experiments on all of the five positional encoding methods. 
Besides, because the data size of UCF-101 is not too large, using a smaller model is more appropriate.
All the details are shown in Table \ref{tab:config-3d-cls}.

\begin{table}[h]
    \centering
    \begin{tabular}{c|c}
        \toprule
        \textbf{Key} & \textbf{Value} \\
        \midrule
        Layers & 8 \\
        Image Size & 224 \\
        Frame Count & 8 \\
        Patch Size & 16 \\
        Hidden Dimension & 384 \\
        Attention Heads & 8 \\
        Batch Size & 768 \\
        Optimizer & AdamW \\
        Weight Decay & 0.01 \\
        Learning Rate & $1.2\times10^{-4}$ \\
        LR Scheduler & cosine \\
        Warmup Ratio & 0.02 \\
        Epochs & 80 \\
        \bottomrule
    \end{tabular}
    \caption{Model and training configuration of 3D classification experiments.}
    \label{tab:config-3d-cls}
\end{table}
    \section{Reformulation of baseline RoPE methods in detail}
\label{app:reformulation}

\begin{table*}[th]
    \centering
    \begin{tabular}{cccc}
        \toprule
        \textbf{Positional Encoding Method} & \textbf{Commutativity} & \textbf{Extra Parameters} & \textbf{Extra Time Complexity} \\
        \midrule
        APE & -- & $nd$ & $O(nd)$ \\
        Vanilla RoPE & Yes & 0 & $O(Lnd(bN+b^2+\frac{d}{h}))\approx O(\frac{Lnd^2}{h})$ \\
        LieRE & Commonly Not & $LNdb$ & $O(Lnd(bN+b^2+\frac{d}{h}))$ \\
        \method-AP (ours) & Yes & $Ldb$ & $O(Lnd(bN+b^2+\frac{d}{h}))$ \\
        \method-LD (ours) & Yes & $Ld(b + \frac{N}{b})$ & $O(Lnd(bN+b^2+\frac{d}{h}))$ \\
        \bottomrule
    \end{tabular}
    \caption{Comparison of different types of positional encoding methods. $n$ represents for count of patches (tokens), $d$ represents for dimension of hidden states, $L$ represents for count of layers, $b$ represents for block size, $N$ represents for count of axes, and $h$ represents the count of attention heads.}
    \label{tab:types-of-pe}
\end{table*}

\subsection{Vanilla RoPE}

Firstly, note that we can represent a 2D rotation matrix in the exponential form:
\begin{equation}
\left(       
  \begin{array}{cc}  
    \mathrm{cos}(\alpha) & -\mathrm{sin}(\alpha) \\
    \mathrm{sin}(\alpha) & \mathrm{cos}(\alpha)
  \end{array}
\right) 
=
\exp (\left(       
  \begin{array}{cc}  
    0 & -\alpha \\
    \alpha & 0
  \end{array}
\right) )
\end{equation}
The solution proposed by RoFormer, which we call vanilla RoPE here, can be regarded as a special type of \method-AP with block size $2$ and non-trainable $\mathbf P_j$ in Equation \ref{eq:rope-ap} where:
\begin{equation}
\begin{split}
\mathbf P_j & =  \left(       
  \begin{array}{cc}  
    \mathrm{cos}(m\theta^{\frac{2N}{d} \cdot j}) & -\mathrm{sin}(m\theta^{\frac{2N}{d} \cdot j}) \\
    \mathrm{sin}(m\theta^{\frac{2N}{d} \cdot j}) & \mathrm{cos}(m\theta^{\frac{2N}{d} \cdot j})
  \end{array}
\right) \\ 
& = \exp(m\theta^{\frac{2N}{d} \cdot j}\left(       
  \begin{array}{cc}  
    0 & -1 \\
    1 & 0
  \end{array}
\right))
\end{split}
\end{equation}
In practice, RoFormer adopts $\theta = 10000^{-1}$ as the hyper-parameter of the rotation base.

\subsection{LieRE}


For LieRE, the blocks are independent and trainable. Hence, we directly define $B_{ij}$ in Equation~\ref{eq:definition-of-matrix-b} as:
\begin{equation}
    \mathbf{B}_{ij} = \mathbf{P}_{ij} - \mathbf{P}_{ij}^\top,
\end{equation}
where $\mathbf{P}_{ij}$ is a trainable matrix.

\section{Analysis and comparison of complexity and extra consumption}

\label{app:complexity}

Table~\ref{tab:types-of-pe} presents an overview of the properties of the positional encoding methods evaluated in this work. Specifically, the table highlights their commutativity (\ie, the commutativity of angle matrices when represented in the RoPE form parameterized by angle matrices), the number of additional parameters, and the extra time complexity introduced by the positional encoding module.

\subsection{APE}

For a Transformer that takes $n$ embeddings with $d$ features as inputs, the extra parameters of position encoding are the tensors in the position code book, i.e., $n \times d$.
The extra computation is to add position embeddings onto the original features. Therefore, the extra time complexity is $O(n \times d)$.

\subsection{RoPE parameterized by angle matrices}

We unify RoPE with angle matrices whose rotation process is presented in Algorithm \ref{alg:rotation}, where $n, h, d, b, N$ represents sequence length, number of heads, dimension of hidden states, block size, and number of axes respectively. In this part, we focus on extra parameters and time complexity on each layer.

\begin{algorithm}
    \caption{Rotation of query and key matrices}
    \label{alg:rotation}
    \begin{algorithmic}
        \STATE \textbf{In 1:} query matrix $\mathbf Q$ with shape $(n, h, \frac{d}{h})$
        \STATE \textbf{In 2:} key matrix $\mathbf K$ with shape $(n, h, \frac{d}{h})$
        \STATE \textbf{In 3:} angle base matrix $\mathbf A$ with shape $(N, h, \frac{d}{hb}, b, b)$
        \STATE \textbf{In 4:} patch positions $\mathbf P$ with shape $(n, N)$
        \STATE \textbf{Out:} rotated query and key matrices $\hat{\mathbf  Q}, \hat{\mathbf  K}$

        \vspace{.6em}
        
        \FOR{$axis = 1$ to $N$}
            \STATE $\mathbf  M_{axis} \gets \mathbf  A_{axis}\odot \mathbf  P_{axis}$
        \ENDFOR
        \STATE $\mathbf  M \gets \sum\mathbf  M_{axis}$ where $M$ has a shape of $(n, h, \frac{d}{hb}, b, b)$
        \STATE $\mathbf  R \gets \text{diag}(e^{\mathbf M}, \text{dim}=2)$ with shape $(n, h, \frac{d}{h}, \frac{d}{h})$
        \STATE $\mathbf{\hat Q} \gets  \mathbf R  \mathbf Q$, $\mathbf{\hat K} \gets \mathbf  R \mathbf  K$

        \vspace{.6em}
        
        \RETURN $\mathbf{\hat Q}, \mathbf{\hat K}$
    \end{algorithmic}
\end{algorithm}

Angle base matrix $\mathbf A$ is defined by the RoPE method, and the extra parameters are brought by the definition of $\mathbf A$.
Time complexity of 
1) calculating the element-wise product over each axis is $O(n\times h \times \frac{d}{hb} \times b^2) = O(ndb)$;
2) calculating sum of $M$ is $O(N\times n \times h \times \frac{d}{hb} \times b^2) = O(ndbN)$;
3) calculating matrix exponential is $O(n\times h \times \frac{d}{hb} \times b^3 = O(ndb^2)$ based on \cite{MatrixExponential};
4) applying rotation is $O(n\times h \times (\frac{d}{h})^2) = O(\frac{nd^2}{h})$.
Thus, the overall time complexity of rotation is $O(ndbN+ndb^2+\frac{nd^2}{h})$.

\subsubsection{Vanilla RoPE}

No extra parameters are presented in vanilla RoPE, and the angle base matrix $\mathbf A$ can be calculated during pre-processing.
Besides, in vanilla RoPE, block size $b=2$, so $\frac{d}{h} \gg bN + b^2 = 2N + 4$ in most cases.
Thus, count of extra parameters are $0$ and extra time complexity is  $O(ndbN+ndb^2+\frac{nd^2}{h})\approx O(\frac{nd^2}{h})$ where $b=2$.

\subsubsection{LieRE}

For LieRE, the angle base matrix can be formulated as $\mathbf A =\mathbf P -\mathbf P^\top$ where the parameters in $\mathbf P$ are all independent. 
The only extra step to get $\mathbf A$ from $\mathbf P$ is the subtraction whose time complexity is $O(Ndb)$.
Thus, count of extra parameters are $N\times h \times \frac{d}{hb} \times b^2 = Ndb$ and extra time complexity is  $O(ndbN+ndb^2+\frac{nd^2}{h} + ndb) = O(ndbN+ndb^2+\frac{nd^2}{h})$.

\subsubsection{\method-AP}

For \method-AP, we compose the angle base matrix $\mathbf A$ whose shape is $(N, h, \frac{d}{hb}, b, b)$ with matrices with shape $(N, h, \frac{d}{hbN}, b, b)$ by filling the blocks that are irrelevant to the corresponding coordinate axes with zeros.
Thus, similarly, count of extra parameters are $N\times h \times \frac{d}{hbN} \times b^2 = db$ and extra time complexity is $O(ndbN+ndb^2+\frac{nd^2}{h})$.

\begin{figure*}[t]
    \centering
    \begin{subfigure}{0.48\linewidth}
        \centering
        \includegraphics[width=\linewidth]{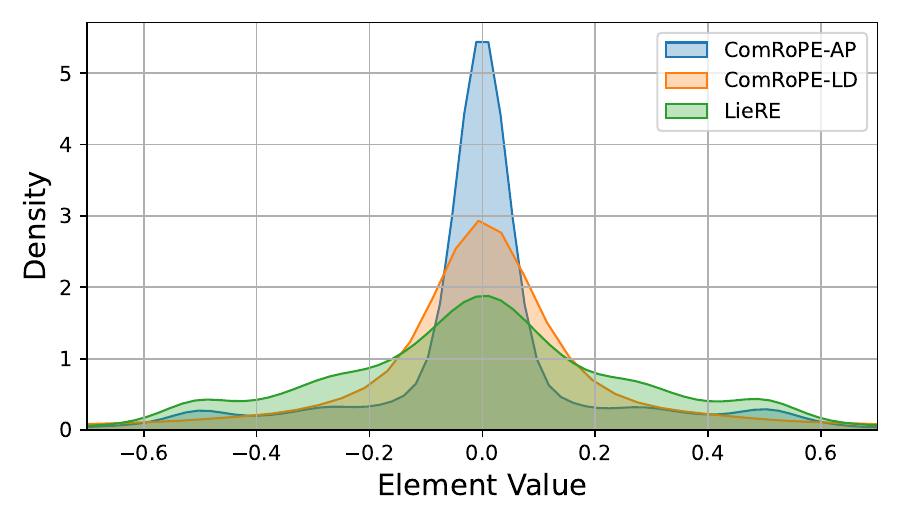}
        \vspace{-2em}
        \caption{Block Size = 2, linear scale}
    \end{subfigure}
    \begin{subfigure}{0.48\linewidth}
        \centering
        \includegraphics[width=\linewidth]{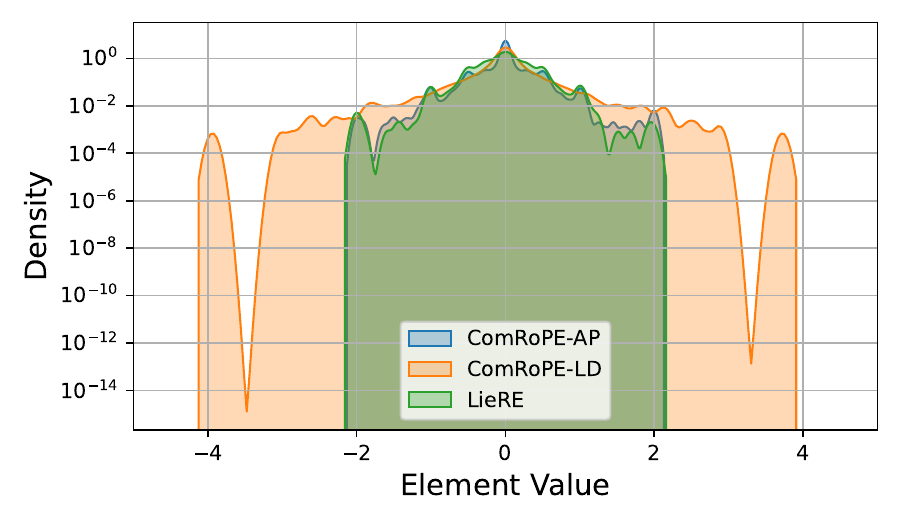}
        \vspace{-2em}
        \caption{Block Size = 2, log scale}
    \end{subfigure}
    \begin{subfigure}{0.48\linewidth}
        \centering
        \includegraphics[width=\linewidth]{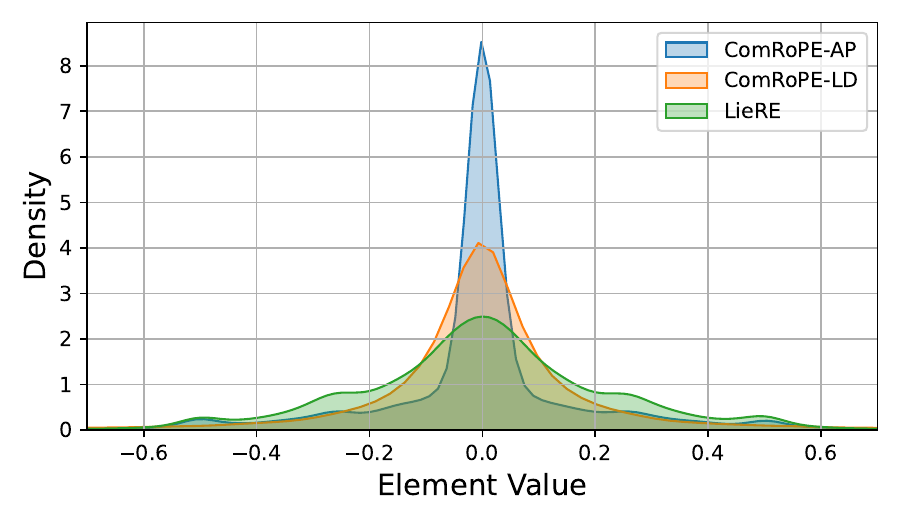}
        \vspace{-2em}
        \caption{Block Size = 4, linear scale}
    \end{subfigure}
    \begin{subfigure}{0.48\linewidth}
        \centering
        \includegraphics[width=\linewidth]{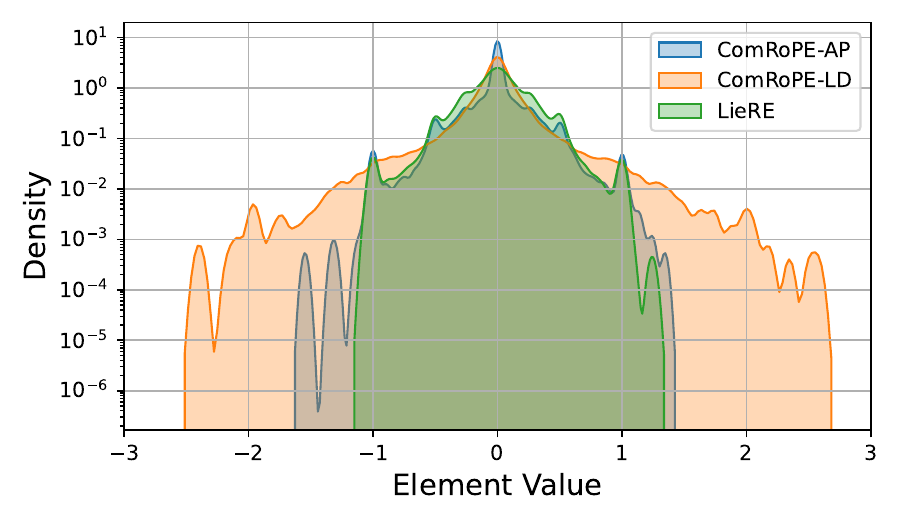}
        \vspace{-2em}
        \caption{Block Size = 4, log scale}
    \end{subfigure}
    \begin{subfigure}{0.48\linewidth}
        \centering
        \includegraphics[width=\linewidth]{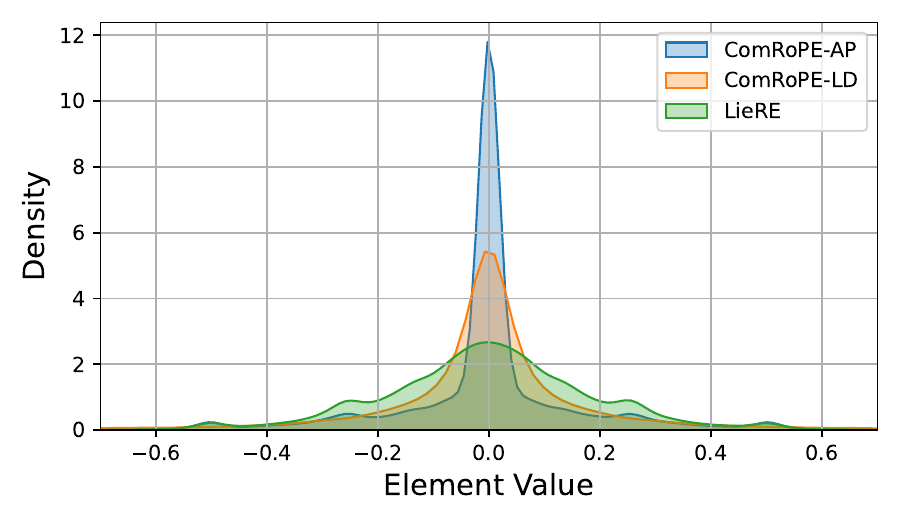}
        \vspace{-2em}
        \caption{Block Size = 8, linear scale}
    \end{subfigure}
    \begin{subfigure}{0.48\linewidth}
        \centering
        \includegraphics[width=\linewidth]{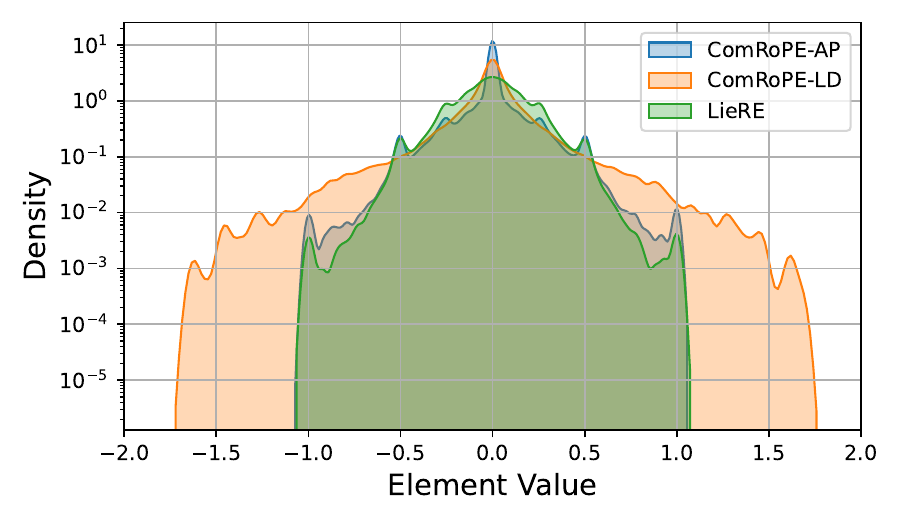}
        \vspace{-2em}
        \caption{Block Size = 8, log scale}
    \end{subfigure}
    \caption{Density distribution of elements in the upper triangular sections of angle matrices from 2D classification experiments. Subfigures (a-b), (c-d), and (e-f) show the distributions for different block sizes: 2, 4, and 8, respectively.}
    \label{fig:2d-cls-dist}
\end{figure*}

\subsubsection{\method-LD}

For \method-LD, the angle base matrices in $\mathbf A$ are pairwise linearly dependent on the first dimension (i.e., axis dimension).
Therefore, it can be presented by a matrix with shape $(h, \frac{d}{hb}, b, b)$ and a multiplication factor with shape $(N, h, \frac{d}{hb})$ by a multiplication step with time complexity $O(N\times h\times \frac{d}{hb}\times b^2) = O(Ndb)$.
Thus, count of extra parameters are $h \times \frac{d}{hb} \times b^2 + N \times h \times \frac{d}{hb} = d(b + \frac{N}{b})$, and extra time complexity is $O(ndbN+ndb^2+\frac{nd^2}{h})$.

\section{Distribution of elements in angle matrices}

\begin{table}[t]
    \centering
    \begin{tabular}{c|c|c}
    \toprule
       Method & \makecell[c]{Block\\Size} & \makecell[c]{Standard\\Deviations} \\ \midrule
        LieRE & \multirow{3}{*}{2} & 0.326 \\
        \method-AP & &  0.271 \\
        \method-LD & &  0.384 \\ \midrule
        LieRE & \multirow{3}{*}{4} & 0.246 \\
        \method-AP & & 0.208 \\
        \method-LD & & 0.278 \\ \midrule
        LieRE & \multirow{3}{*}{8} & 0.195 \\
        \method-AP & & 0.171 \\
        \method-LD & & 0.238 \\ \bottomrule
    \end{tabular}
    \caption{The standard deviations of elements in angle matrices obtained from the 2D classification experiments.}
    \label{tab:2d-cls-dist}
\end{table}

In this section, we analyze the element distribution in angle matrices obtained from the 2D classification experiments. Specifically, we extract all elements from the upper triangular parts of the matrices. The standard deviations of these elements are summarized in Table \ref{tab:2d-cls-dist}, and their density plot is presented in Figure \ref{fig:2d-cls-dist}.

To provide a clearer view of the long-tail distribution, we present the density plot using both linear and logarithmic scales in Figure \ref{fig:2d-cls-dist}. From the linear scale plot, it can be observed that elements near zero exhibit the highest variance in the angle matrices of LieRE, while \method-AP demonstrates the most moderate variance. On the other hand, the logarithmic scale reveals notable differences in range. For instance, \method-LD retains a broader distribution at values farther from zero. Consequently, as indicated in Table \ref{tab:2d-cls-dist}, \method-LD exhibits the largest overall variance among the angle matrix elements. This phenomenon is likely due to the linear dependencies between angle matrices across different coordinate axes, which necessitate significant frequency differences to distinguish them effectively.

    \section{More Analysis}

\noindent\textbf{Computational complexity and time consumption.}
Time consumption is shown in Table~\ref{tab:comp_cost}. While small block sizes should have minimal impact, parallel optimization issues in \texttt{torch.matrix\_exp} lower GPU utilization, increasing time costs for LieRE and ComRoPE.

\begin{table}[h]
    \centering
    \resizebox{0.98\linewidth}{!}{
        \begin{tabular}{c|cccccc}
            \toprule
            Method & \makecell[c]{Vanilla\\RoPE} & \makecell[c]{LieRE\\block = $8\times8$} & \makecell[c]{ComRoPE-LD\\block = $4\times4$} & \makecell[c]{ComRoPE-AP\\block = $4\times4$} & \makecell[c]{ComRoPE-LD\\block = $8\times8$} & \makecell[c]{ComRoPE-AP\\block = $8\times8$} \\ \midrule
            Training Time per Epoch (min) & 16 & 21 & 19 & 19 & 20 & 20 \\
            Inference Time on Valid Split (s) & 32 & 36 & 35 & 34 & 35 & 35 \\
            \bottomrule
        \end{tabular}
    }
    \caption{Computational costs compared on A800$\times 4$.}
    \label{tab:comp_cost}
    \vspace{-0.9em}
\end{table}

\noindent\textbf{Application to LLMs.}
ComRoPE can be incorporated into large language models as a \emph{drop-in} substitute for the rotary position embeddings used in most pre-trained checkpoints, requiring no extra architectural changes during fine-tuning.
Because language modeling operates along a single sequence dimension, the commuting property of our angle matrices holds automatically.
At present, however, the \texttt{torch.matrix\_exp} implementation incurs substantial memory overhead on large models, making end-to-end training prohibitively expensive.
Addressing this bottleneck, thereby unlocking full-scale LLM experiments, remains a key priority for future work.

\noindent\textbf{Implementation with sota codebase and settings.}
Our work compares RoPE designs under consistent settings to highlight relative advantages, as demonstrated by our experiments. 
For a more thorough and  convincing comparison, we conduct additional experiments in the RoPE-Mixed codebase with DeiT data augmentation (C.f. Table~\ref{tab:deit_results}).

\begin{table}[h]
    \centering
    \resizebox{0.98\linewidth}{!}{

\begin{tabular}{c|cccccccc}
\toprule
\multirow{2}{*}{\textbf{\makecell[c]{Position Encoding\\Method}}} & \multicolumn{8}{c}{\textbf{Evaluation Resolution}}             \\

  & \textbf{128}   & \textbf{192}   & \textbf{224}   & \textbf{256}   & \textbf{320}   & \textbf{384}   & \textbf{448}   & \textbf{512}   \\ \midrule
\textbf{RoPE-Mixed}                                & {\ul 68.99} & 79.75          & 81.42          & 82.31          & {\ul 82.75}    & {\ul 82.11}    & {\ul 80.61}    & 78.39          \\
\textbf{ComRoPE-AP}                                & 68.48          & \textbf{80.94} & \textbf{82.01} & {\ul 82.59}    & 82.43          & 81.65          & 80.58          & \textbf{79.75} \\
\textbf{ComRoPE-LD}                                & \textbf{69.88}    & {\ul 79.91}    & {\ul 81.78}    & \textbf{83.24} & \textbf{83.36} & \textbf{82.32} & \textbf{80.79} & {\ul 78.97}    \\ 
\bottomrule
\end{tabular}

    }
    \caption{Results with DeiT recipe. RoPE-Mixed corresponds to ComRoPE-LD with $2\times2$ blocks. ComRoPE-LD consistently outperforms RoPE-Mixed.}
    \label{tab:deit_results}
    \vspace{-0.9em}
\end{table}

\section{Limitations}

Despite the merits of our approach, it has two notable constraints that call for further investigation.  
The first one is computational overhead. Our implementation depends on \texttt{torch.matrix\_exp}, which is slow and memory-intensive on large models. Cutting training time and GPU memory use is therefore an urgent engineering goal.
And the other is strict commutativity restrictions. We currently require relatively strong conditions for the angle matrices to commute, which may restrict the expressiveness of the resulting embeddings. Identifying weaker—yet still sufficient—conditions could broaden the method’s capacity and applicability.  
Addressing these two issues will be the cornerstone of our future work, paving the way for more efficient training and richer modeling flexibility.

}


\end{document}